\documentclass{article}

\usepackage{microtype}
\usepackage{graphicx}
\usepackage{subcaption}
\usepackage{booktabs} 

\usepackage[preprint]{icml2026}

\usepackage{amsmath}
\usepackage{amssymb}
\usepackage{mathtools}
\usepackage{amsthm}

\usepackage[colorlinks,
		linkcolor = blue,		
		anchorcolor = blue,	
		urlcolor = blue,		
		citecolor = blue,		
]{hyperref}
\usepackage{booktabs}       
\usepackage{graphicx}
\usepackage{cases}
\usepackage{bm}
\usepackage{bbm}
\usepackage{multirow}
\usepackage[dvipsnames]{xcolor} 
\usepackage{colortbl}
\usepackage{listings}
\usepackage{wrapfig}
\usepackage{makecell}
\usepackage{enumitem}
\usepackage{tcolorbox}
\usepackage{svg}

\usepackage[capitalize,noabbrev]{cleveref}

\theoremstyle{plain}
\newtheorem{theorem}{Theorem}[section]
\newtheorem{proposition}[theorem]{Proposition}
\newtheorem{lemma}[theorem]{Lemma}
\newtheorem{corollary}[theorem]{Corollary}
\theoremstyle{definition}

\theoremstyle{remark}

\usepackage[textsize=tiny]{todonotes}

\icmltitlerunning{FourierMoE: Fourier Mixture-of-Experts Adaptation of Large Language Models}

\begin{document}

\twocolumn[
  \icmltitle{FourierMoE: Fourier Mixture-of-Experts Adaptation of Large Language Models}



  \icmlsetsymbol{equal}{*}
  \icmlsetsymbol{corres}{\dag}




  \begin{icmlauthorlist}
  \icmlauthor{Juyong Jiang}{hkustgz,hkust,equal}
  \icmlauthor{Fan Wang}{hkustgz,equal}
  \icmlauthor{Hong Qi}{hkustgz}
  \icmlauthor{Sunghun Kim}{hkustgz}
  \icmlauthor{Jing Tang}{hkustgz,hkust,corres}
  \end{icmlauthorlist}
    
  \icmlaffiliation{hkustgz}{The Hong Kong University of Science and Technology (Guangzhou)}
  \icmlaffiliation{hkust}{The Hong Kong University of Science and Technology}
    
  \icmlcorrespondingauthor{Jing Tang}{jingtang@ust.hk}

  \icmlkeywords{Large Language Models, Fourier Transform, Mixture-of-Experts, Parameter-Efficient Fine-Tuning}

  \vskip 0.3in
]



\printAffiliationsAndNotice{\icmlEqualContribution}

\begin{abstract}
Parameter-efficient fine-tuning (PEFT) has emerged as a crucial paradigm for adapting large language models (LLMs) under constrained computational budgets. However, standard PEFT methods often struggle in multi-task fine-tuning settings, where diverse optimization objectives induce task interference and limited parameter budgets lead to representational deficiency. While recent approaches incorporate mixture-of-experts (MoE) to alleviate these issues, they predominantly operate in the spatial domain, which may introduce structural redundancy and parameter overhead. To overcome these limitations, we reformulate adaptation in the spectral domain. Our spectral analysis reveals that different tasks exhibit distinct frequency energy distributions, and that LLM layers display heterogeneous frequency sensitivities. Motivated by these insights, we propose \textbf{FourierMoE}, which integrates the MoE architecture with the inverse discrete Fourier transform (IDFT) for frequency-aware adaptation. Specifically, FourierMoE employs a frequency-adaptive router to dispatch tokens to experts specialized in distinct frequency bands. Each expert learns a set of conjugate-symmetric complex coefficients, preserving complete phase and amplitude information while theoretically guaranteeing lossless IDFT reconstruction into real-valued spatial weights. Extensive evaluations across 28 benchmarks, multiple model architectures, and scales demonstrate that FourierMoE consistently outperforms competitive baselines in both single-task and multi-task settings while using significantly fewer trainable parameters. These results highlight the promise of spectral-domain expert adaptation as an effective and parameter-efficient paradigm for LLM fine-tuning.
\end{abstract}

\begin{figure*}[t]
\centering
\includegraphics[width=0.8\linewidth]{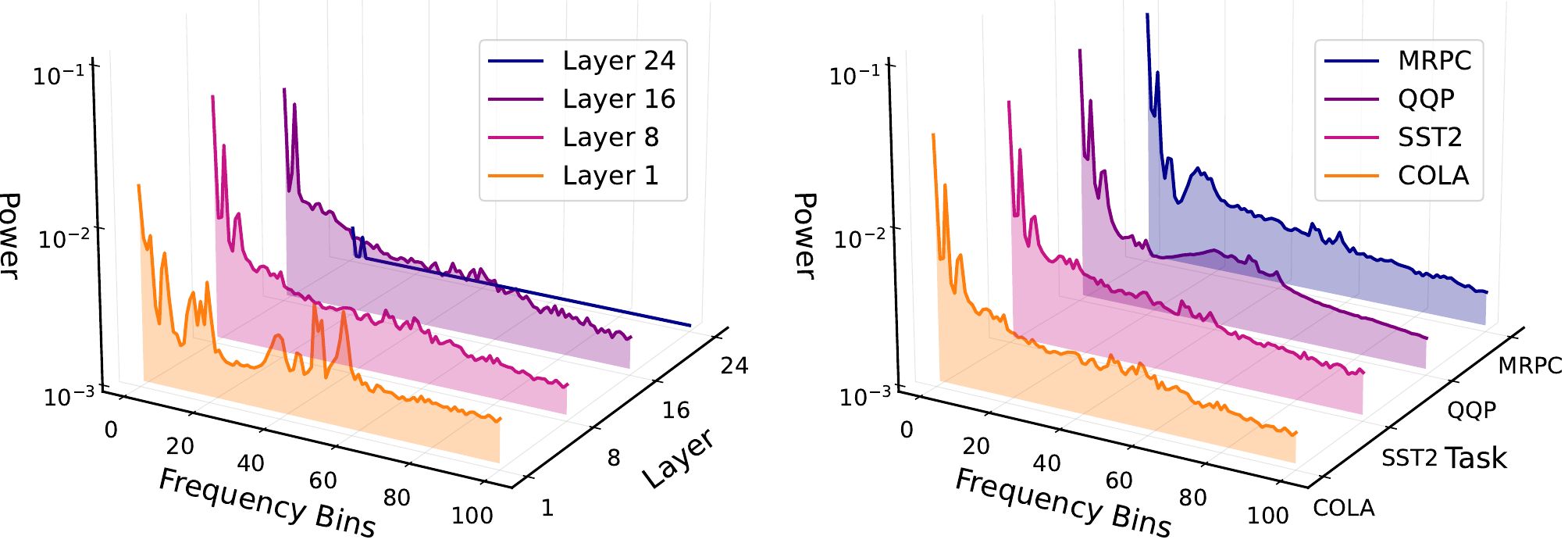}
\caption{Spectral analysis across layers and tasks. \textbf{(Left)} The power spectral density of the RoBERTa-large weights shows layer-wise differences, with early layers exhibiting clear high-frequency spikes and deeper layers showing a progressively smoother spectrum. 
\textbf{(Right)} For different GLUE tasks (CoLA, SST-2, QQP, and MRPC), the spectra of hidden representations from the eighth layer of RoBERTa-large display distinct frequency energy distributions, revealing task-specific preferences.
These observations suggest that effective adaptation requires frequency-specific modulation tailored to different layers and downstream tasks.}
\label{fig:intro_exp1}
\vspace{-1.1em}
\end{figure*}

\section{Introduction}
\label{sec:intro}
The advent of large language models (LLMs) \citep{dubey2024llama,team2024gemma,yang2025qwen3} has reshaped the landscape of natural language processing (NLP). However, as model sizes escalate to hundreds of billions of parameters, full fine-tuning (FFT) becomes increasingly prohibitive due to exorbitant computational and memory overheads. Parameter-efficient fine-tuning (PEFT) \citep{hu2021lora,liu2024dora} has emerged as a crucial paradigm, adapting pretrained models by updating a minimum set of parameters \citep{zi2023delta,lialin2023scaling}. 

While methods like LoRA \citep{hu2021lora} demonstrate efficacy in single-task adaptation, their performance often degrades in multi-task scenarios \citep{li2024mixlora,tian2024hydralora}. We attribute this degradation to two primary challenges: task interference and representation deficiency. Given the data heterogeneity inherent in multi-task learning, distinct tasks often dictate conflicting optimization objectives \citep{yu2020gradient}. When forcing diverse tasks to share a monolithic set of adaptable parameters, gradient updates can exhibit orthogonal or opposing directions, leading to negative transfer \citep{zhang2025more,liu2024moe}.
Furthermore, the restricted parameter budget of the conventional single-PEFT methods, which rely on a single tunable module, may limit their capacity to model the fine-grained and diverse structural features required for simultaneous generalization across multiple tasks \citep{valipour2023dylora,liu2024dora,zhang2022adaptive,park2025llamaduo}. 

This challenge in balancing parameter efficiency with representational capacity for multi-task generalization has motivated recent advances in more dynamic and compositional architectures. A promising direction integrates the mixture-of-experts (MoE) architecture with PEFT, leading to mixture of parameter-efficient experts (MoPE) approaches \citep{zadouri2023pushing,dou2023loramoe}. These methods utilize a router to dynamically select experts instantiated as PEFT modules conditioned on inputs, enabling a flexible allocation of task-specific capacity and potentially reducing task interference \citep{cai2024survey}. Despite their promise, existing MoPE methods suffer from structural redundancy \citep{tian2024hydralora,gao2024higher} due to the lack of explicit mechanisms to encourage orthogonality or diversity among experts. Furthermore, maintaining multiple spatial experts incurs additional parameter overhead, which partially compromises the efficiency goals of PEFT.

In this work, we diverge from the spatial-centric convention and explore solutions through the lens of frequency. We conduct a spectral analysis, as shown in Figure \ref{fig:intro_exp1}, revealing heterogeneity in frequency sensitivity across both model layers and downstream tasks. While early layers exhibit high-frequency spikes, deeper layers manifest a progressively smoother spectrum with attenuated high-frequency components, suggesting a transition from local feature extraction to global semantic integration. Furthermore, different tasks exhibit distinct frequency energy distributions. For instance, MRPC displays a notable energy plateau in the mid-frequency range (bins 15-40), while SST-2 demonstrates a relatively smoother decay. These observations imply that effective adaptation requires frequency-specific modulation tailored to layers and tasks, whereas uniform adaptation may be inefficient and suboptimal.

Building upon these insights, we propose \textbf{FourierMoE} (\textbf{\underline{Fourier}} \textbf{\underline{M}}ixture-\textbf{\underline{o}}f-\textbf{\underline{E}}xperts), a novel method that integrates MoE architecture with the inverse discrete Fourier transform (IDFT) for flexible, frequency-aware adaptation. Our approach incorporates a frequency-adaptive router and a set of experts specialized in distinct spectral bands. Each expert learns a small set of conjugate-symmetric complex coefficients, enabling effective adaptation in the Fourier domain while theoretically ensuring reconstruction into real-valued spatial weights. FourierMoE achieves strong expressivity and theoretical soundness along three dimensions: (1) It partitions the spectrum into distinct frequency bands to facilitate expert specialization. (2) Prior spectral PEFT methods \citep{gaoparameter,kim2025lfma} instantiate the learnable coefficients as real-valued, but we argue that ignoring the imaginary component limits the ability to represent phase and amplitude information. In contrast, FourierMoE learns both real and imaginary components to enable a complete spectral representation. (3) By enforcing conjugate symmetry on the coefficients, our method theoretically ensures that IDFT reconstruction yields real-valued updates, maintaining consistency with the spatial model weights while avoiding information loss.

We validate the effectiveness of FourierMoE across 28 benchmarks spanning commonsense reasoning, math reasoning, image classification, and natural language understanding (NLU). Our method achieves state-of-the-art (SOTA) results compared with FFT, PEFT, and MoPE baselines in both single-task and multi-task setups \citep{li2024mixlora}, while using significantly fewer trainable parameters. 
Our main contributions are summarized as follows:
\begin{itemize}
    \item We propose FourierMoE, a novel MoPE approach that employs frequency-specialized experts with a frequency-adaptive router to facilitate flexible and fine-grained adaptation for LLMs.
    \item We provide a theoretical analysis of spectral PEFT, uncovering the representation limitations of real-only coefficient learning and asymmetric frequency sampling in prior methods \citep{gaoparameter, kim2025lfma}. Derived from this analysis, FourierMoE's design choice is theoretically grounded to address these limitations.
    \item Extensive experiments show that FourierMoE consistently outperforms competitive baselines across various model architectures and scales on 28 benchmarks in both single-task and multi-task scenarios, thereby validating its effectiveness and robustness.
\end{itemize}
\section{Related Work} \label{sec:short_related_work}
PEFT has emerged as a crucial paradigm for adapting LLMs to downstream tasks \citep{lialin2023scaling,han2024parameter}, with numerous methods consistently improving efficiency or adaptation quality \citep{hu2021lora,liu2024dora,zhang2022adaptive,dettmers2023qlora,meng2024pissa}.
Recently, the spectral domain has emerged as a PEFT frontier. FourierFT \citep{gaoparameter} and LFMA \citep{kim2025lfma} learn Fourier coefficients and map them back to the spatial model weights via IDFT. However, they rely on static coefficient positions and neglect the imaginary components, restricting expressive capacity. FourierMoE addresses these constraints by combining MoE with carefully structured coefficients, enhancing flexibility and representational power. Further discussions of related works are presented in Appendix \ref{sec:related_work}.
\section{Methodology}
\label{sec:methodology}
\begin{figure*}[t]
\centering
\includegraphics[width=0.9\linewidth]{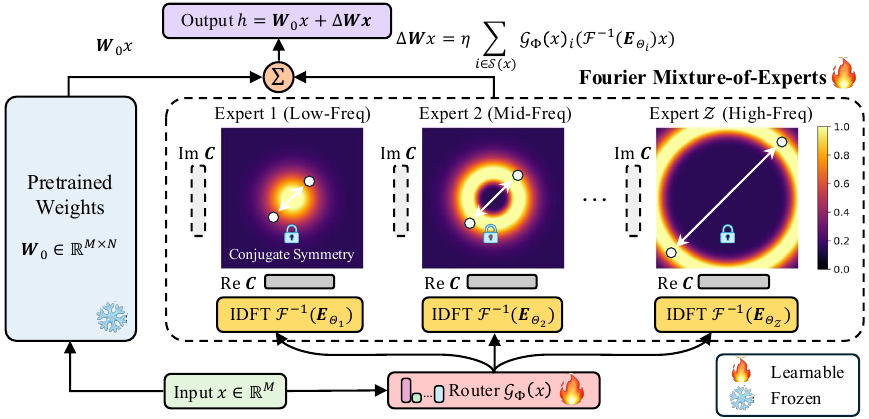}
\caption{The overall framework of FourierMoE, which reparameterizes LLM weight updates $\Delta \mathbf{W}$ in the spectral domain. A frequency-adaptive router $\mathcal{G}_{\Phi}(x)$ dynamically assigns tokens to experts specialized in distinct frequency bands, which mitigates task interference via parameter isolation. Each expert learns conjugate-symmetric complex coefficients, enabling complete spectral representation while theoretically guaranteeing real-valued weight updates after IDFT.}
\label{fig:fouriermoe}
\vspace{-0.6em}
\end{figure*}

In this section, we formally introduce FourierMoE, a novel MoPE framework that reparameterizes the weight adaptation of LLMs into the spectral domain. Our approach is motivated by the observation that LLMs exhibit heterogeneous frequency sensitivities across layers and downstream tasks. By leveraging the orthogonality of the Fourier basis \citep{davis2012fourier} and the divide-and-conquer principle of MoE \citep{masoudnia2014mixture}, FourierMoE achieves fine-grained, frequency-aware optimization. We first outline the spectral reparameterization formulation, followed by the design of frequency-specialized experts, and finally provide a rigorous theoretical analysis of the conjugate symmetry constraints required for lossless adaptation.

\subsection{Spectral Reparameterization}
\label{subsec:spectral_reparam}
Motivated by the low intrinsic dimension of pretrained LLMs \citep{aghajanyan2021intrinsic}, LoRA \citep{hu2021lora} models $\Delta \mathbf{W}$ using a low-rank structure in the spatial domain. In the spectral domain, the informational signals within neural network weights exhibit spectral sparsity, with the important adaptation information concentrated in a small subset of dominant frequency components \citep{kim2025lfma}. In addition, FourierFT empirically shows that sparse spectral coefficients can recover high-quality weight updates \citep{gaoparameter}. Based on these insights and observations, we assume that learning $\Delta \mathbf{W}$ within a sparse spectral subspace is sufficiently expressive for downstream adaptation. 

Let $\mathbf{W}_0 \in \mathbb{R}^{M \times N}$ denote the frozen pretrained weights. We aim to learn an update $\Delta \mathbf{W}$ such that the forward pass becomes $h = (\mathbf{W}_0 + \Delta \mathbf{W})x$. We model any spatial weight update $\Delta \mathbf{W}$ as the inverse discrete Fourier transform (IDFT) of a spectral signal $\mathbf{F} \in \mathbb{C}^{M \times N}$. The transformation from the spectral domain back to the spatial domain is defined as:
\begin{align}
\label{eq:reparam}
\Delta \mathbf{W} = \mathcal{F}^{-1}(\mathbf{F}) &= \frac{1}{MN} \sum_{u=0}^{M-1}\sum_{v=0}^{N-1} \mathbf{F}(u,v) \cdot \mathbf{B}_{u,v}, \nonumber \\
\mathbf{B}_{u,v}(q, y) &= \exp\!{\left(j 2\pi \left( \frac{uq}{M} + \frac{vy}{N} \right)\right)}, \nonumber \\
q \in \{0, \ldots, & M-1\}, \,y \in \{0, \ldots, N-1\} 
\end{align}
where $\mathbf{B}_{u,v} \in \mathbb{C}^{M \times N}$ represents the 2D Fourier basis kernel. Unlike prior approaches \citep{gaoparameter,kim2025lfma} that instantiate $\mathbf{F}(u,v)$ as real-valued, we strictly model $\mathbf{F}(u,v)$ in the complex domain $\mathbb{C}$.

\subsection{Fourier Mixture-of-Experts}

To address the multi-task conflict and limited capacity of existing single-PEFT methods, FourierMoE utilizes a sparse ensemble of $Z$ experts gated by a router $\mathcal{G}_{\Phi}(x)$ parameterized by $\Phi$. 
Instead of aggregating experts in the frequency domain, we transform each expert's spectral representation $\mathbf{E}_i \in \mathbb{C}^{M \times N}$ into the spatial domain to obtain expert-specific updates $\Delta \mathbf{W}_i$. To ensure computational efficiency, we employ a token-level Top-$k$ routing strategy, activating only the subset of experts with the highest gating scores. The final composite update is formulated as:
\begin{equation}
\Delta \mathbf{W} = \sum_{i \in \mathcal{S}(x)} \mathcal{G}_{\Phi}(x)_i \cdot \Delta \mathbf{W}_i, \quad \Delta \mathbf{W}_i = \mathcal{F}^{-1}(\mathbf{E}_i),
\end{equation}
where $\mathcal{S}(x)$ denotes the set of $k$ indices corresponding to the top-$k$ elements of the router output, and $\mathcal{G}_{\Phi}(x)_i$ represents the gating weights for the selected $i$-th experts. This formulation ensures that the gating mechanism selects among the realized spatial weights, allowing for dynamic and input-dependent adaptation.

Notably, although FourierMoE adopts token-level routing \citep{lepikhin2020gshard,fedus2022switch}, the IDFT is not recomputed per token. Specifically, each active expert's spatial update, $\Delta \mathbf{W}_i = \mathcal{F}^{-1}(\mathbf{E}_i)$, is reconstructed once per layer in each forward pass and then reused for all tokens assigned to that expert. Therefore, the reconstruction cost scales with the number of active experts $k$, rather than with the number of tokens.

\paragraph{Band-Limited Spectral Experts.}
Each expert $\mathbf{E}_i$ is designed to specialize in a particular frequency band, capturing distinct features ranging from global semantics (low-frequency) to local syntactic variations (high-frequency). An expert is defined by a learnable parameter set $\Theta_i$ comprising $n$ active frequency coordinates defined by an index set $\Omega_i = \{(u_k, v_k)\}_{k=1}^n$. 
For each active frequency $(u,v) \in \Omega_i$, we learn a complex coefficient $\mathbf{C}_{i}(u,v) = a_{u,v} + j b_{u,v}$. The necessity of learning both real and imaginary components is proven below.

\begin{proposition}[\textbf{Phase-Amplitude Completeness}]
Restricting spectral coefficients to $\mathbb{R}$ (i.e., $b_{u,v}=0$) forces the phase $\Phi_{u,v} = \operatorname{atan2}(b, a)$ to be either $0$ or $\pi$. This constrains the resulting spatial signal to be an even function (symmetric around the origin), rendering the model incapable of representing spatial shifts or asymmetric features in $\Delta \mathbf{W}$. By learning $\mathbf{C}_{i} \in \mathbb{C}$, FourierMoE maintains full expressivity over amplitude $A = \sqrt{a^2+b^2}$ and phase $\Phi$.
\end{proposition}

\paragraph{Gaussian Spectral Initialization.}
To enforce spectral specialization, we utilize a Gaussian bandpass filter \citep{filter} to initialize the active indices $\Omega_i$. The probability of assigning a frequency coordinate $(u,v)$ to expert $i$ is governed by:
\begin{equation}\label{eq:gaussian}
P_i(u,v) = \exp\bigg( - \bigg( \frac{\mathcal{D}(u,v)^2 - \mathcal{D}_{c,i}^2}{\mathcal{D}(u,v) \cdot \mathcal{W}_i} \bigg)^2 \bigg),
\end{equation}
where $\mathcal{D}(u,v)$ is the Euclidean distance from the DC component (origin), $\mathcal{D}_{c,i}$ is the center frequency, and $\mathcal{W}_i$ is the bandwidth. This mechanism can ensure minimal spectral overlap, allowing the router to dispatch tokens based on their required frequency resolution.

\subsection{Theoretical Analysis}
\label{subsec:theory}

Given that the LLM weights $\mathbf{W}_0$ and the input $x$ are real-valued, the resulting update $\Delta \mathbf{W}$ must theoretically lie in $\mathbb{R}^{M \times N}$. Moreover, since $\Delta \mathbf{W}$ is a linear combination of expert updates $\Delta \mathbf{W}_i$, each expert update must likewise be real-valued. We provide the theoretical guarantee that ensures $\Delta \mathbf{W}_i \in \mathbb{R}$ without heuristic truncation.

\begin{theorem}[\textbf{Conjugate Symmetry Condition}]
\label{thm:symmetry}
Let $\mathbf{F} \in \mathbb{C}^{M \times N}$ be a spectral matrix (representing any expert $\mathbf{E}_i$). The inverse discrete Fourier transform $\mathbf{S} = \mathcal{F}^{-1}(\mathbf{F})$ satisfies $\mathbf{S} \in \mathbb{R}^{M \times N}$ (i.e., $\mathrm{Im}(\mathbf{S}) = \mathbf{0}$) if and only if $\mathbf{F}$ satisfies Hermitian symmetry:
\begin{equation}
\mathbf{F}(u,v) = \mathbf{F}^*(\langle -u \rangle_M, \langle -v \rangle_N), \quad \forall (u,v),
\end{equation}
where $\langle \cdot \rangle_M$ denotes the modulo $M$ operation and $(\cdot)^*$ denotes the complex conjugate.
\end{theorem}

\begin{proof}
Consider the IDFT definition for an entry $\mathbf{S}(q,y)$:
\begin{align}
\mathbf{S}(q,y) &= \frac{1}{MN} \sum_{u=0}^{M-1}\sum_{v=0}^{N-1} \mathbf{F}(u,v) e^{j \frac{2\pi uq}{M}} e^{j \frac{2\pi vy}{N}}.
\end{align}
Let $\theta_{u,v} = 2\pi(\frac{uq}{M} + \frac{vy}{N})$. We split the summation into the DC component, the Nyquist components (if any), and pairs of indices $(u,v)$ and their reflection $(u', v') = (\langle -u \rangle_M, \langle -v \rangle_N)$.
Assume the symmetry condition holds: $\mathbf{F}(u', v') = \mathbf{F}^*(u,v)$. The contribution of this pair to the sum is:
\begin{align}
T &= \mathbf{F}(u,v)e^{j\theta_{u,v}} + \mathbf{F}(u',v')e^{j\theta_{u',v'}} \\
  &= \mathbf{F}(u,v)e^{j\theta_{u,v}} + \mathbf{F}^*(u,v)e^{j(2\pi k - \theta_{u,v})} \, (k \in \mathbb{Z}) \\
  &= \mathbf{F}(u,v)e^{j\theta_{u,v}} + \left(\mathbf{F}(u,v)e^{j\theta_{u,v}}\right)^* \\
  &= 2 \operatorname{Re}\left\{ \mathbf{F}(u,v)e^{j\theta_{u,v}} \right\}.
\end{align}
Since $2 \operatorname{Re}\{\cdot\}$ is strictly real, the total sum $\mathbf{S}(q,y)$ comprises only the real terms. Conversely, if symmetry is violated, the imaginary terms fail to cancel, resulting in $\mathbf{S}(q,y) \in \mathbb{C}$.
\end{proof}

\paragraph{Necessity of Symmetry Constraints.}
Prior works \citep{gaoparameter,kim2025lfma} often neglect the conjugate symmetry, computing $\Delta \mathbf{W}_{trunc} = \operatorname{Re}(\mathcal{F}^{-1}(\mathbf{F}_{unsym}))$. We demonstrate that this leads to a representation error.

\begin{corollary}[\textbf{Truncation Error Bound}]
Let $\mathbf{F}_{unsym}$ be a spectral matrix violating Theorem \ref{thm:symmetry}. The effective weight update for the model is $\mathbf{W}_{eff} = \operatorname{Re}(\mathcal{F}^{-1}(\mathbf{F}_{unsym}))$. The information loss due to the truncation of imaginary parts, defined as the energy of the discarded signal, is given by:
\begin{align}
\mathcal{L}_{error} &= \| \mathcal{F}^{-1}(\mathbf{F}_{unsym}) - \mathbf{W}_{eff} \|_F^2 \nonumber \\
& = \sum_{q,y} \left( \operatorname{Im}(\mathcal{F}^{-1}(\mathbf{F}_{unsym}))_{q,y} \right)^2.
\end{align}
By Parseval's theorem \citep{oppenheim1999discrete}, the truncation error can be equivalently characterized in the spectral domain, up to a DFT-dependent normalization constant:
\begin{equation}
\mathcal{L}_{error} \propto \sum_{u,v} \| \mathbf{F}_{unsym}(u,v) - \mathbf{F}^*_{unsym}(\langle -u \rangle_M, \langle -v \rangle_N) \|^2.
\end{equation}
\end{corollary}
Without enforcing conjugate symmetry, updates fall into imaginary subspace that are discarded, causing information loss.
FourierMoE enforces the constraint in Theorem \ref{thm:symmetry} for each expert $\mathbf{E}_i$ during the learning of coefficients $\Theta_i$, ensuring the update contributed by each expert is strictly real-valued, eliminating the truncation error.

\subsection{Optimization Objective}
The forward pass for the $l$-th layer is formulated as:
\begin{equation}\label{eq:optimization}
h = \mathbf{W}_0 x + \eta \sum_{i \in \mathcal{S}(x)} \mathcal{G}_{\Phi}(x)_i \left(\mathcal{F}^{-1}(\mathbf{E}_{\Theta_i}) x\right),
\end{equation}
where $\eta$ denotes a predefined scaling factor, $\Theta = \{\Theta_i\}_{i=1}^Z$ represents the ensemble of learnable spectral coefficients for all experts, and $\Phi$ denotes the trainable parameters of the frequency-adaptive router $\mathcal{G}$. 
We optimize the joint parameter set $\{\Theta, \Phi\}$ by minimizing a composite objective function that balances task performance with architectural stability. Formally, the optimization problem is defined as the minimization of the expected loss over the training distribution $\mathcal{D}_{train}$:
\begin{align}\label{eq:loss}
\min_{\Theta, \Phi} \mathbb{E}_{x \sim \mathcal{D}_{train}} \left[ \mathcal{L}_{task}(x; \Theta, \Phi) + \lambda \mathcal{L}_{aux}(x; \Phi) \right],
\end{align}
where $\mathcal{L}_{task}$ is the standard autoregressive cross-entropy loss. To mitigate the risk of expert collapse, a common failure mode in sparse MoE where the router over-selects a small subset of experts, we define the auxiliary load-balancing loss $\mathcal{L}_{aux}$ following prior works \citep{lepikhin2020gshard,fedus2022switch}:
\begin{align}\label{eq:aux_loss}
\mathcal{L}_{aux} = Z \sum_{i=1}^Z f_i \cdot P_i, \quad f_i = \frac{1}{B} \sum_{x \in \mathcal{B}} \mathbbm{1}\{i \in \mathcal{S}(x)\}
\end{align}
where $f_i$ is the fraction of tokens in a batch $\mathcal{B}$ assigned to expert $i$ (i.e., expert $i$ is in the top-$k$ set for $x$), and $P_i$ is the mean routing probability for expert $i$ across the batch. This formulation balances expert optimization for broad spectral coverage and diverse representations. For clarity, we present the detailed algorithm in Appendix \ref{sec:algorithm}.
\section{Experiments}\label{sec:experiments}
\subsection{Experiment Setting}
\textbf{Baselines.} We compare FourierMoE against FFT, single-PEFT, and MoPE baselines. For a fair and comprehensive evaluation, we follow the configurations used in prior works and reuse their reported results. We select competitive single-PEFT methods as baselines, including LoRA \citep{hu2021lora}, rsLoRA \citep{kalajdzievski2023rank}, DoRA \citep{liu2024dora}, PiSSA \citep{meng2024pissa}, MiLoRA \citep{wang2024milora}, KaSA \citep{wang2024kasa}, LoRA-Dash \citep{si2025unleashing}, NEAT \citep{zhong2024neat}, TopLoRA \citep{li2025beyond}, MELoRA \citep{ren2024melora}, and FourierFT \citep{gaoparameter}. In addition, we include comparisons against several MoPE baselines: MoLoRA \citep{zadouri2023pushing}, AdaMoLE \citep{liu2024adamole}, HydraLoRA \citep{tian2024hydralora}, and GOAT \citep{fan2025make}. The detailed introduction of baselines is presented in Appendix \ref{sec:baselines}.

\textbf{Model and Datasets.} We evaluate the efficiency of FourierMoE across multi-task and single-task scenarios spanning NLP and computer vision (CV) domains. For the multi-task setup \citep{li2024mixlora}, we train models on mixed tasks and evaluate them on the individual test set, including: (1) Commonsense reasoning: we fine-tune LLaMA-2 7B \citep{touvron2023llama2} and Gemma 7B \citep{team2024gemma} on the Commonsense170K dataset \citep{hu2023llm}, and evaluate on the test sets of its constituent subsets.
(2) Math reasoning: we fine-tune LLaMA-3 8B \citep{grattafiori2024llama} and Qwen2.5-14B \citep{yang2025qwen3} on Math10K \citep{hu2023llm}, and evaluate them on GSM8K \citep{cobbe2021training}, SVAMP \citep{patel2021nlp}, MultiArith \citep{roy2016solving}, AddSub \citep{hosseini2014learning}, AQuA \citep{ling2017program}, and SingleEq \citep{koncel2015parsing}.
For the single-task setup, we train and evaluate models on each task, including: (3) Image classification: we fine-tune and evaluate CLIP ViT-B/32 \citep{radford2021learning} on Cars \citep{krause20133d}, DTD \citep{cimpoi2014describing}, EuroSAT \citep{helber2019eurosat}, GTSRB \citep{houben2013detection}, RESISC45 \citep{cheng2017remote}, SUN397 \citep{xiao2010sun}, and SVHN \citep{netzer2011reading}.
(4) Natural language understanding: we fine-tune and evaluate RoBERTa-large \citep{liu2019roberta} on the CoLA, SST-2, MRPC, QQP, MNLI, QNLI, and RTE subsets from GLUE \citep{wang2018glue}. We present detailed statistics and hyperparameter configurations for all datasets and benchmarks in Appendix \ref{sec:statistics_dataset} and Appendix \ref{sec:train_details}. All experiments are conducted on NVIDIA H100 (80GB) GPUs.  

\begin{table*}[ht]
\caption{Performance comparison of single-PEFT and MoPE methods on eight commonsense reasoning benchmarks using LLaMA-2 7B and Gemma 7B in a multi-task setting. Accuracy is reported for all benchmarks. \textbf{Bold} denotes the best result.}
\label{tab:commonsense}
\scriptsize
\centering
\resizebox{\linewidth}{!}{%
\begin{tabular}{lccccccccccc}
\toprule
\textbf{Model} & \textbf{Method} & \textbf{\# Params(\%)} & \textbf{BoolQ} & \textbf{PIQA} & \textbf{SIQA} & \textbf{HellaSwag} & \textbf{WinoGrande} & \textbf{ARC-e} & \textbf{ARC-c} & \textbf{OBQA} & \textbf{Average} \\
\midrule
ChatGPT & / & / & 73.10 & 85.40 & 68.50 & 78.50 & 66.10 & 89.80 & 79.90 & 74.80 & 77.01 \\
\midrule
\multirow{11}{*}{LLaMA-2 7B}
& LoRA & 0.84 & 69.80 & 79.90 & 79.50 & 83.60 & 82.60 & 79.80 & 64.70 & 81.00 & 77.61 \\
& DoRA & 0.84 & 71.80 & 83.10 & 79.90 & 89.10 & 83.00 & 84.50 & 71.00 & 81.20 & 80.45 \\
& PiSSA & 0.84 & 67.60 & 78.10 & 78.40 & 76.60 & 78.00 & 75.80 & 60.20 & 75.60 & 73.79 \\
& MiLoRA & 0.84 & 67.60 & 83.80 & 80.10 & 88.20 & 82.00 & 82.80 & 68.80 & 80.60 & 79.24 \\
& LoRA-Dash & 0.84 & 71.00 & 75.70 & 79.30 & 91.10 & 78.60 & 84.20 & 69.80 & 78.80 & 78.56 \\
& NEAT & 0.84 & 71.70 & 83.90 & 80.20 & 88.90 & 84.30 & 86.30 & 71.40 & 83.00 & 81.21 \\
& KaSA  & 0.84 & 73.60 & 84.40 & 80.20 & 91.50 & 84.50 & 84.70 & 72.10 & 81.20 & 81.53 \\
& HydraLoRA & 0.84 & 72.78	&84.06	&79.68	&80.34	& 86.66	&87.12	&72.35	&86.00 &81.12 \\
& MoLoRA & 0.96 & 73.15	&83.68	&80.09	&74.57	&85.95	&87.33	&72.53 & 86.20 & 80.43 \\
& GOAT & 0.96 & 73.60 & 83.95	& 80.50	& 87.12	& 85.00	& 87.79 &76.88 & 87.00 & 82.73 \\
\rowcolor[RGB]{211, 244, 225}
& \textbf{FourierMoE} & 0.06 & \textbf{73.73} & \textbf{84.60} & \textbf{81.27} & \textbf{92.33} & \textbf{86.82} & \textbf{88.24} & \textbf{77.21} & \textbf{87.40} & \textbf{83.95} \\
\midrule
\multirow{5}{*}{Gemma 7B}
& LoRA & 0.14 & 75.17 & 88.74 & 77.58 & 95.21 & 89.11 & 92.93 & 84.73 & 88.00 & 86.43 \\
& DoRA & 0.14 & 73.49 & 90.44 & 79.19 & 95.03 & 90.00 & 93.79 & 84.73 & 88.67 & 86.92 \\
& MELoRA & 0.14 & 73.49 & 89.50 & 79.99 & 94.60 & 89.90 & 93.18 & 84.30 & 89.40 & 86.79 \\
& HydraLoRA & 0.16 & 72.14 & 89.21 & 81.30 & 95.11 & 89.45 & 94.56 & 85.32 & 89.00 & 87.01\\ 
\rowcolor[RGB]{211, 244, 225} & \textbf{FourierMoE} & 0.03 & \textbf{75.26} & \textbf{91.05} & \textbf{82.55} & \textbf{95.81} & \textbf{90.21} & \textbf{95.44} & \textbf{85.62} & \textbf{89.60} & \textbf{88.19}  \\
\bottomrule
\end{tabular}
}
\end{table*}
\begin{table*}[!t]
    \vspace{0.2em}
    \centering 
    \caption{Performance comparison of single-PEFT and MoPE methods on six math reasoning benchmarks using LLaMA-3 8B and Qwen2.5-14B in a multi-task setting. Accuracy is reported for all benchmarks. \textbf{Bold} denotes the best result.}
    \resizebox{0.9\textwidth}{!}{
    \begin{tabular}{lccccccccccc}
        \toprule
        \textbf{Model} & \textbf{Method} & \textbf{\# Params(\%)} & \textbf{AddSub} & \textbf{AQuA} & \textbf{GSM8K} & \textbf{MultiArith} & \textbf{SingleEQ} & \textbf{SVAMP} & \textbf{AVG.} \\ 
        \midrule 
         \multirow{5}{*}{LLaMA-3 8B}
         & LoRA & 0.12 & 84.56 & 25.72 & 57.22 & 91.22 & 92.26 & 70.17 & 70.19 \\
         & DoRA & 0.12 & 85.95 & 26.19 & 56.52 & 89.67 & \textbf{92.62} & 70.40 & 70.23 \\
         & MELoRA & 0.12 & 85.82 & 24.41 & 55.34 & 87.83 & 91.54 & 71.20 & 69.36 \\
         & HydraLoRA & 0.11 & 86.08 & 25.98 & 55.50 & 91.00 & 91.14 & 68.10 & 69.63 \\ 
         \rowcolor[RGB]{211, 244, 225}
         & \textbf{FourierMoE} & 0.01 & \textbf{87.85} & \textbf{33.07} & \textbf{60.80} & \textbf{93.16}  & 91.54  & \textbf{73.00} & \textbf{73.24} \\ 
         \midrule
         \multirow{6}{*}{Qwen2.5-14B}
         & LoRA & 0.12 & 91.90 & 34.65 & 74.37 & 96.33 & 92.91 & 86.40 & 79.43 \\
         & DoRA & 0.13 & 92.41 & 36.22 & 75.92 & 96.39 & 92.39 & 86.80 & 80.02 \\
         & TopLoRA & 0.10 & 91.31 & 35.96 & 77.31 & \textbf{97.67} & 93.44 & 87.43 & 80.52 \\
         & MELoRA & 0.12 & 92.91 & 33.86 & 75.89 & 97.33 & 92.13 & 85.60 & 79.62 \\
         & HydraLoRA & 0.12 & 92.41 & 36.61 & 76.32 & 96.22 & 92.45 & 86.97 & 80.16 \\
         \rowcolor[RGB]{211, 244, 225}
         & \textbf{FourierMoE} &  0.05 & \textbf{94.68} & \textbf{41.34} & \textbf{78.09} & 97.33  & \textbf{95.87}  & \textbf{91.80} & \textbf{83.19} \\
        \bottomrule
    \end{tabular}}
    \label{tab:math_reasoning}
    \vspace{-0.275em}
\end{table*}
\subsection{Main Results}
\textbf{Commonsense Reasoning.} As shown in Table \ref{tab:commonsense}, FourierMoE consistently surpasses the baselines across all benchmarks and models, while using minimal trainable parameters. On LLaMA-2 7B, it outperforms the strongest MoPE baseline, GOAT, by 1.22\%, while using 16$\times$ fewer trainable parameters. On Gemma 7B, our method achieves the best performance of 88.19 with 0.03\% of parameters, demonstrating its superior efficiency in the multi-task setup. 

\textbf{Math Reasoning.} We evaluate FourierMoE on larger LLMs to validate its scalability. As presented in Table \ref{tab:math_reasoning}, our method consistently attains the highest average accuracy across both backbone models. On LLaMA-3 8B, it improves accuracy by 3.01\% over the strongest single-PEFT baseline (DoRA) and by 3.61\% over the best MoPE method (HydraLoRA). It also achieves notable gains over the second-best results on AQuA (+6.88\%), GSM8K (+3.58\%), and SVAMP (+1.8\%). On Qwen2.5-14B, FourierMoE achieves SOTA performance across 5/6 benchmarks, demonstrating effectiveness and robust scalability across models.

\begin{table*}[ht] 
\caption{Performance comparison of FFT, single-PEFT, and MoPE methods on seven image classification benchmarks using CLIP ViT-B/32 in a single-task setting. Accuracy is reported for all benchmarks. \textbf{Bold} denotes the best result.}
\scriptsize
\centering
\resizebox{\linewidth}{!}{%
\begin{tabular}{lcccccccccc}
\toprule
\textbf{Method} & \textbf{\# Params (\%)} & \textbf{Cars} & \textbf{DTD} & \textbf{EuroSAT} & \textbf{GTSRB} & \textbf{RESISC45} & \textbf{SUN397} & \textbf{SVHN} & \textbf{Average} \\
\midrule
FFT  & 100 & 60.33 & 73.88 & 98.96 & 98.30 & 93.65 & 53.84 & 96.78 & 82.25 \\
FFT MoE & 770 & 66.39 & 75.53 & 98.59 & 98.50 & 94.38 & 60.34 & 97.09 & 84.40 \\
\midrule
LoRA (rank8)  & 1.49 & 41.02 & 70.15 & 98.66 & 96.51 & 90.38 & 47.51 & 95.39 & 77.09 \\
LoRA (rank16) & 2.99 & 46.51 & 72.07 & {98.74} & 98.04 & 92.08 & 51.63 & 96.00 & 79.30 \\
LoRA (rank32) & 5.98 & 50.13 & 72.87 & 98.88 & 98.13 & 92.87 & 53.65 & 96.55 & 80.44 \\
DoRA & 1.49 & 40.75 & 71.91 & 98.89 & 97.71 & 90.19 & 47.54 & 95.46 & 77.49  \\
PiSSA & 1.49 & 40.41 & 69.62 & 98.48 & 95.84 & 90.58 & 47.21 & 95.84 & 76.85 \\
MiLoRA & 1.49 & 39.77	& 70.48	& 98.19	& 97.52 &	89.92 &	45.38 & 95.49 & 76.68 \\
FourierFT & 0.22 & 39.20 & 70.58 & 98.02 & 97.93 & 91.58 & 61.92 & 91.14 & 78.62 \\
MoLoRA & 2.24  & 50.83 & {73.51} & 98.63 & 97.72 & 92.58 & 52.55 & 96.00 & 80.26 \\
AdaMoLE & 2.33 & 49.47 & 71.65 & 98.52 & 97.73 & 91.95 & 52.29 & 95.82 & 79.63 \\
HydraLoRA & 1.58 & 48.42 & 72.18 & 98.40 & 97.28 & 92.93 & 51.80 & 96.06 & 79.58  \\
GOAT &2.24& 53.50 & 75.32 & {98.82} & 98.17 & 93.46 & 	54.53	& 96.62 & 81.49 \\
\rowcolor[RGB]{211, 244, 225}
\textbf{FourierMoE} & 0.42 & \textbf{64.44} & \textbf{77.18} & \textbf{98.93} & \textbf{98.28} & \textbf{93.65} & \textbf{66.41} & \textbf{96.78} & \textbf{85.10} \\
\bottomrule
\end{tabular}
}
\label{tab:image_classification}
\end{table*}
\textbf{Image Classification.} Table \ref{tab:image_classification} demonstrates that FourierMoE surpasses all PEFT and MoPE baselines across seven CV benchmarks. Notably, it outperforms FFT by 2.85\% and FFT MoE by 0.7\%, while reducing trainable parameters from 100\% and 770\% to just 0.42\%, significantly lowering adaptation costs. The consistent improvements underscore the versatility and cross-modal generalizability of FourierMoE beyond NLP tasks. In addition, our method consistently outperforms the spatial MoPE baselines, including MoLoRA, AdaMoLE, HydraLoRA, and GOAT, indicating the effectiveness of the spectral reparameterization. 

\begin{table*}[!t] 
\caption{Performance comparison of FFT, single-PEFT, and MoPE methods on seven benchmarks from GLUE using RoBERTa-large in a single-task setting. Accuracy is reported for all benchmarks. \textbf{Bold} denotes the best result.}
\label{tab:nlu}
\scriptsize
\centering
\resizebox{0.9\linewidth}{!}{%
\begin{tabular}{lcccccccccc}
\toprule
\textbf{Method} & \textbf{\# Params (\%)} & \textbf{CoLA} & \textbf{SST-2} & \textbf{MRPC} & \textbf{QQP} & \textbf{MNLI} & \textbf{QNLI} & \textbf{RTE} & \textbf{Average} \\
\midrule
FFT  & 100 & 84.27 & 95.98 & 85.29 & 91.58 & 89.83 & 94.49 & 84.84 & 89.47 \\
FFT MoE & 698 & 86.02	&96.22	&85.05	&92.20	&90.20	&95.10	&84.48 & 89.90 \\
\midrule
LoRA & 4.00 & 83.41 & 95.64 & 83.33 & 90.06 & 89.00 & 93.28 & 84.47 & 88.46 \\
DoRA & 4.00 & {85.33} & 95.99 & 84.07 & 91.24 & {89.52} & 93.54 & 84.48 & 89.17 \\
PiSSA & 4.00 & 69.12 & 95.98 & 82.84 & 91.24 & 88.94 & 93.59 & 73.29 & 85.00 \\
MiLoRA & 4.00 & {84.65} & {96.10} &{86.02} & {91.33} & {89.51} & {94.12} & {84.83} & {89.51} \\
rsLoRA & 4.00 & 83.51 & 95.98 & {86.02} & 90.75 & 88.97 & 93.84 & 84.12 & 89.03 \\
FourierFT & 0.42 & 78.01 & 95.20 & 86.27 & 85.48 & 87.19 & 91.59 & 83.39 & 86.73 \\ 
MoLoRA & 4.50 & 83.94 & 96.10 & 87.75 & 91.45 & 89.36 & 93.90 & 84.11 & 89.52 \\
AdaMoLE  & 4.56 & 83.99 & 95.76 & 86.03 & 91.48 & 89.21 & 93.64 & 83.75 & 89.12 \\
GOAT & 4.50 & \textbf{86.86} & 96.21 & 84.55 & 91.40 & 89.55 & 94.19 & 85.56 & 89.76 \\
HydraLoRA & 2.75 & 83.89 & 95.52 & 85.04 & 91.02 & 89.34 & 93.87 & 81.22 & 88.56 \\ 
\rowcolor[RGB]{211, 244, 225} \textbf{FourierMoE} & 0.42 & 85.52 & \textbf{96.33} & \textbf{91.91} & \textbf{91.52} & \textbf{90.33} & \textbf{94.66} & \textbf{89.53} & \textbf{91.40} \\
\bottomrule
\end{tabular}
}
\end{table*}
\textbf{Natural Language Understanding.} Table \ref{tab:nlu} shows that our method outstrips all baselines across 6/7 benchmarks using 0.42\% of trainable parameters. It exceeds the best single-PEFT baseline, MiLoRA (89.51), by 1.89\%, and the best MoPE method, GOAT (89.76), by 1.64\%. Moreover, FourierMoE outperforms the spectral PEFT method FourierFT (91.40 vs. 86.73) under the same parameter budget. These outcomes suggest the efficiency of our proposed method in the single-task fine-tuning scenario.

\subsection{In-depth Analysis and Insights}
In the following sections, we present an in-depth analysis of FourierMoE, covering its component contributions, expert scalability, expert frequency, hyperparameter sensitivity, and computation efficiency.

\textbf{Ablation Study.} We assess the contribution of each component of FourierMoE across the Cars, DTD, and SUN397 benchmarks using CLIP ViT-B/32. Specifically, the w/o imaginary part variant learns only the real component of the Fourier coefficients, while the w/o real part variant learns only the imaginary component. The w/o frequency bias variant randomly samples the frequency coordinates for each expert. The w/o conjugate symmetry variant removes the conjugate symmetry constraints. As shown in Figure \ref{fig:ablation_study}, the full implementation of FourierMoE consistently achieves the best performance across the evaluated benchmarks, whereas ablating any component leads to performance degradation, validating the effectiveness of the core design choices.

\textbf{Expert Scalability.} We evaluate the scalability of FourierMoE in terms of coefficient amount, expert amount, and activation count on MRPC, QNLI, and RTE using RoBERTa-large. As illustrated in Figure \ref{fig:expert_scaling} (left), the model performance initially scales with the coefficient count and peaks at $n=1008$, beyond which further increases lead to performance degradation. A similar trend is observed for the total number of experts $Z$ (center), where accuracy improves up to $Z=8$ before declining. Regarding the expert activation count $k$ (right), our results indicate that activating the top-2 experts yields the best performance, while excessively sparse or dense activation tends to be suboptimal. These findings indicate that FourierMoE balances performance with structural sparsity, achieving high parameter efficiency.

\begin{figure*}[t]
	\centering
    \includegraphics[width=0.85\linewidth]{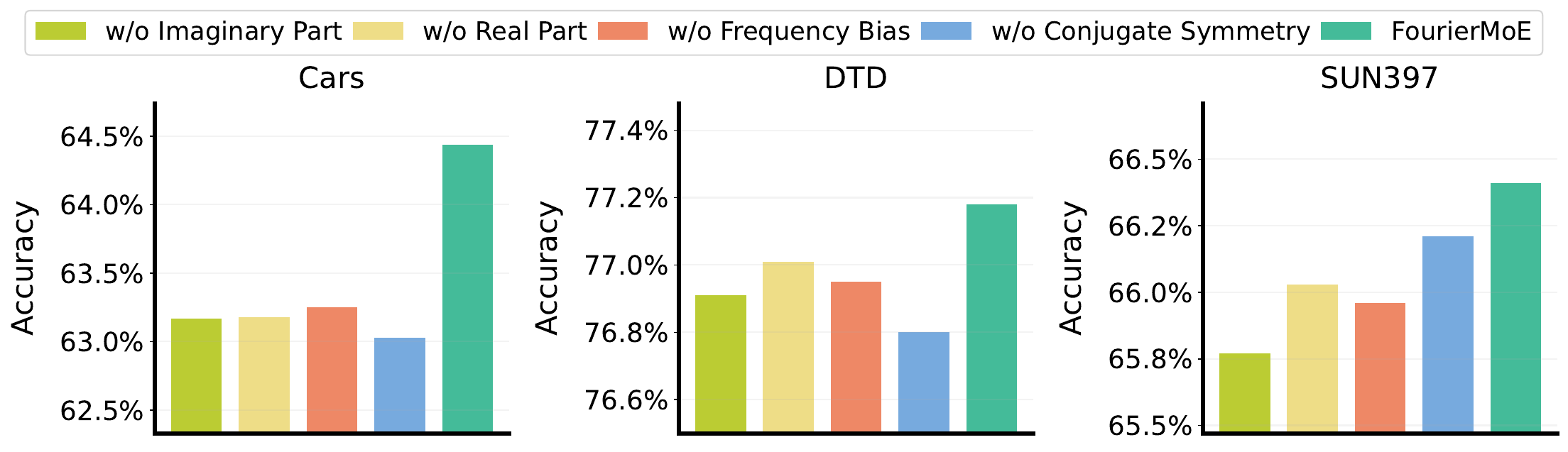}
	\caption{Component ablation study of FourierMoE on Cars, DTD, and SUN397. Accuracy is reported to assess the contribution of each component. Removing any component results in a performance drop compared to the full FourierMoE implementation.}
\label{fig:ablation_study}
\vspace{-0.5em}
\end{figure*}

\begin{figure*}[t]
	\centering
    \includegraphics[width=0.85\linewidth]{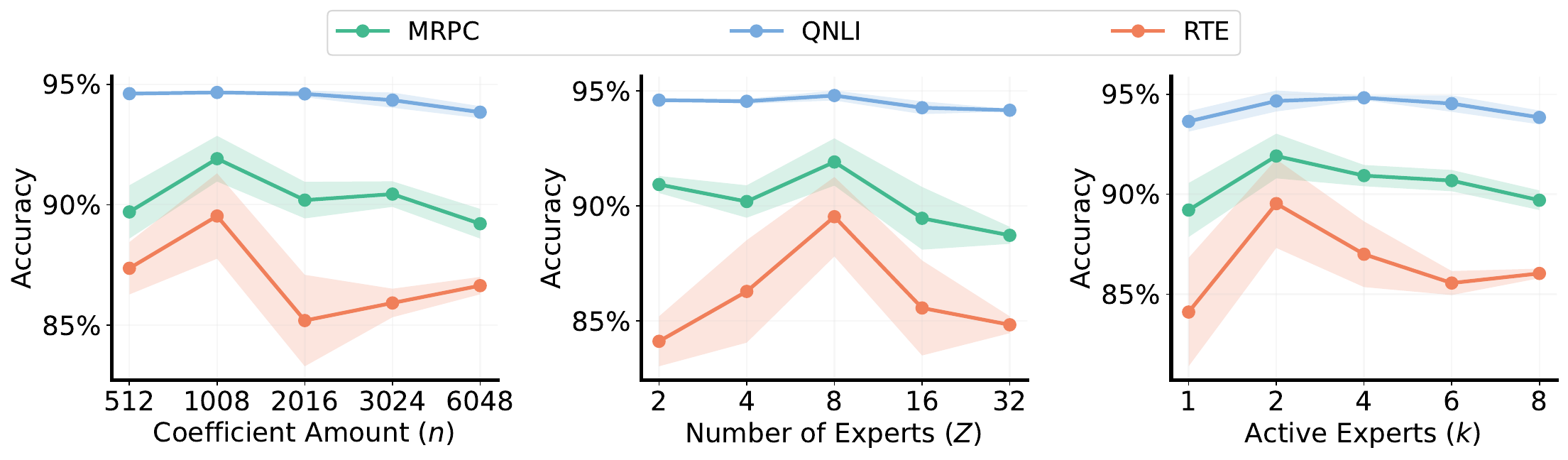}
	\caption{Expert scaling analysis on the MRPC, QNLI, and RTE datasets.
    We report the accuracy scores for each dataset. 
    \textbf{Left:} Impact of the trainable coefficient count $n$ per expert.
    \textbf{Center:} Scalability with respect to the total number of experts $Z$ (fixed activation $k=2$).
    \textbf{Right:} Effect of the active expert count $k$ given a fixed total pool size of $Z=8$.}
\label{fig:expert_scaling}
\vspace{-0.7em}
\end{figure*}

\begin{figure}[h]
    \centering
    \includegraphics[width=\linewidth]{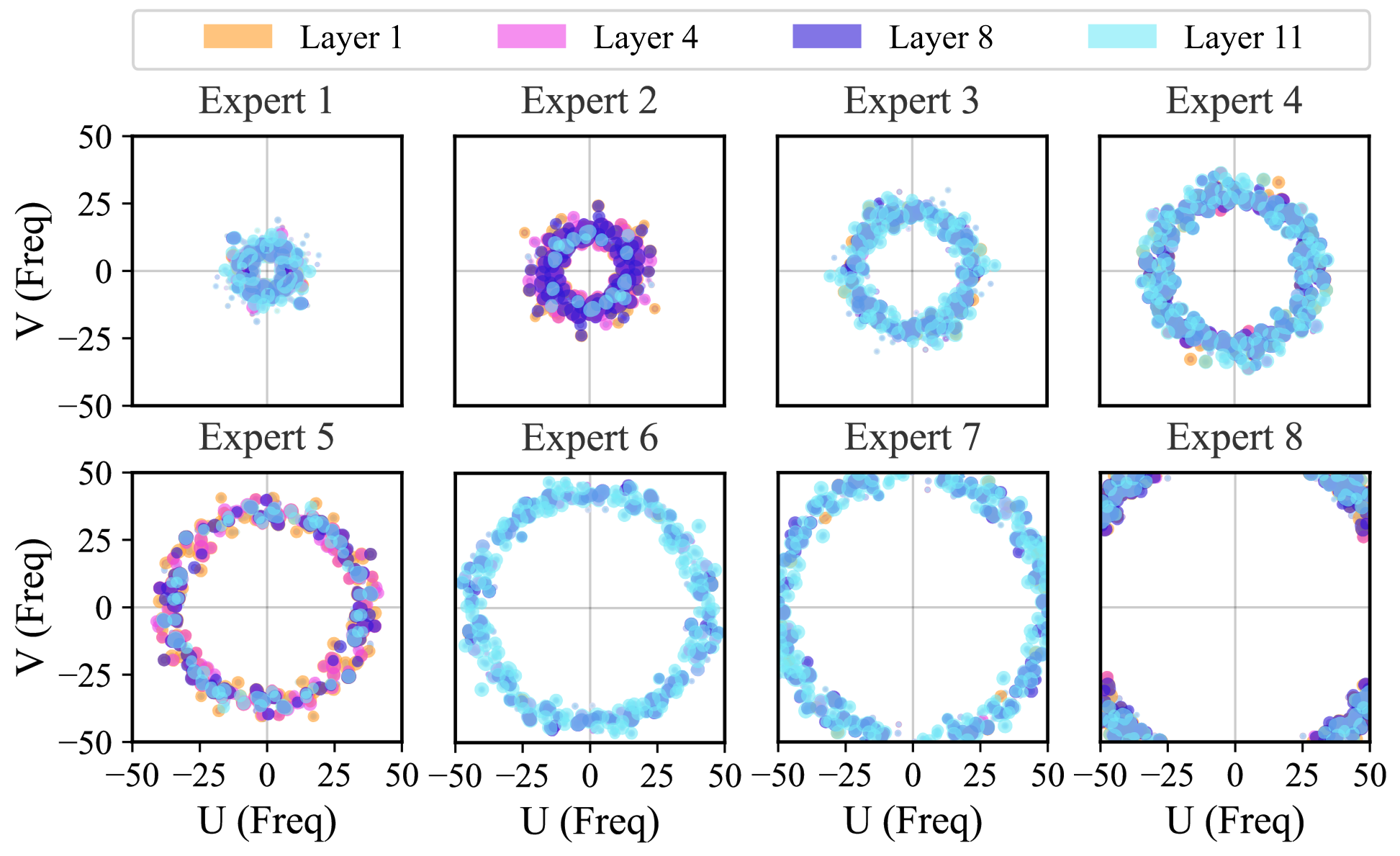}
    \caption{Visualization of expert-wise coefficient distributions across different layers of the fine-tuned model. Each expert's coefficients consistently exhibit concentrations across layers, indicating stable specialization.}
    \label{fig:expert_frequency}
    \vspace{-0.8em}
\end{figure}

\textbf{Expert Frequency.} We visualize the expert-wise coefficient distributions of FourierMoE in Figure \ref{fig:expert_frequency}. Our empirical results reveal that each expert’s learned coefficients consistently cluster within a distinct and coherent spectral region across different layers. This phenomenon suggests that FourierMoE effectively maintains spectral specialization throughout the fine-tuning process. Such specialized behavior not only enhances the model's expressive diversity but also characterizes its structural advantage in navigating the complexities of multi-task fine-tuning scenarios.

\begin{table}[h]
\vspace{-0.2em}
\centering
\caption{Performance comparison of FourierMoE under different frequency bandwidths ($\mathcal{W}$) on CoLA, MRPC, and RTE.}
\label{tab:hyper_band}
\resizebox{0.66\linewidth}{!}{
\begin{tabular}{lccc}
\toprule
\textbf{Setting} & \textbf{CoLA} & \textbf{MRPC} & \textbf{RTE} \\
\midrule
0.06 & 83.03 & 90.44 & 86.28 \\
0.12 & 85.52 & 91.91 & 89.53 \\
0.24 & 83.89 & 89.71 & 87.36 \\
0.48 & 83.31 & 88.48 & 86.28 \\
0.96 & 82.74 & 88.95 & 84.47 \\
\bottomrule
\end{tabular}}
\vspace{-0.9em}
\end{table}

\begin{table}[h!]
\centering
\caption{Performance comparison of FourierMoE using various load balancing loss weight ($\lambda$) on Cars, DTD, and SUN397.}
\label{tab:hyper_balance}
\resizebox{0.67\linewidth}{!}{
\begin{tabular}{lccc}
\toprule
\textbf{Setting} & \textbf{Cars} & \textbf{DTD} & \textbf{SUN397} \\
\midrule
0.001 & 64.44 & 77.18 & 66.41 \\
0.005 & 64.32 & 77.44 & 65.88\\
0.01 & 64.25 & 77.12 & 66.00 \\
0.05 & 64.00 & 76.91 & 65.96\\
0.1 & 64.32 & 76.80 & 65.81\\
\bottomrule
\end{tabular}}
\vspace{-0.9em}
\end{table}

\textbf{Hyperparameter Sensitivity.} \label{sec:hyperparameter_sensitivity}
We evaluate the impact of two core hyperparameters, including the frequency bandwidth $\mathcal{W}$ and load balancing loss weight $\lambda$. As shown in Table \ref{tab:hyper_band}, expanding $\mathcal{W}$ from $0.06$ to $0.12$ improves performance across benchmarks, while a further increase beyond this point causes accuracy degradation. We speculate that this is related to the role of $\mathcal{W}$ in determining the specialization of experts, consequently shaping model performance. We provide a detailed analysis of frequency bandwidth in Appendix \ref{sec:bandpass}. Furthermore, Table \ref{tab:hyper_balance} shows variations in $\lambda$ lead to marginal fluctuations in performance. This suggests that FourierMoE maintains a stable optimization trajectory, indicating that the auxiliary routing regularization does not significantly interfere with the task learning objective.

\begin{table}[h]
\centering
\caption{Efficiency comparison of methods fine-tuning RoBERTa-large on RTE using a single NVIDIA H100 (80GB) GPU. Note that training FLOPs ($\times 10^9$) are calculated per sample, training latency is reported per epoch, and inference latency is measured with a batch size of 32.}
\label{table:computation_efficiency}
\resizebox{\linewidth}{!}{
\begin{tabular}{lcccc}
\toprule
\textbf{Metric} & \textbf{LoRA} & \textbf{TopLoRA} & \textbf{GOAT} & \textbf{FourierMoE} \\
\midrule
\# Trainable Params & 2.38 M & 2.36 M & 16.98 M & 1.49 M \\
\# GPU Memory & 12.24 GB & 16.90 GB & 17.12 GB & 11.61 GB \\
\# Training FLOPs & 340.02 & 258.84 & 163.46 & 235.23 \\
Training Latency & 13.13 s & 145.69 s & 867 s & 58 s \\
Inference Latency & 53.4 ms & 397 ms & 840 ms & 313 ms \\
\midrule
Performance & 82.31 & 85.20 & 85.56 & 89.53 \\
\bottomrule
\end{tabular}
}
\vspace{-0.4em}
\end{table}

\textbf{Computation Efficiency.}
Table \ref{table:computation_efficiency} compares the efficiency of FourierMoE with single-PEFT baselines (LoRA and TopLoRA) and a MoPE baseline (GOAT). While LoRA exhibits low training and inference latency, its performance remains limited (82.31). TopLoRA and GOAT improve model performance over LoRA but introduce a significant increase in both training and inference latencies. In contrast, FourierMoE attains the best performance (89.53) while using only 1.49M trainable parameters, achieving an 11.4$\times$ trainable parameter reduction compared to GOAT. Although the training and inference latencies of FourierMoE are higher than those of LoRA (58 s vs. 13.13 s; 313 ms vs. 53.4 ms), they remain moderate in absolute terms and are still much lower than those of TopLoRA (145.69 s; 397 ms) and GOAT (867 s; 840 ms). These results highlight that our FourierMoE offers a favorable performance-efficiency trade-off.

\section{Conclusion}\label{sec:conclusion}
In this work, we first conduct a spectral analysis that reveals heterogeneity in frequency sensitivity across both model layers and downstream tasks, suggesting the necessity of fine-grained adaptation in the spectral domain. Based on these observations, we propose FourierMoE, a frequency-aware MoPE framework for adapting LLMs. Specifically, FourierMoE departs from spatial-domain adaptation by routing tokens to experts specialized in distinct frequency bands, with each expert learning conjugate-symmetric complex coefficients in the Fourier domain. This design helps reduce inter-expert redundancy and task interference under limited parameter budgets, while encouraging diverse frequency representations and theoretically ensuring real-valued weight updates. Extensive evaluations across 28 benchmarks show that FourierMoE consistently outperforms competitive PEFT and MoPE baselines with superior parameter efficiency, highlighting the potential of spectral modulation for scalable LLM adaptation.

\bibliography{ref}

@article{touvron2023llama2,
  title={Llama 2: Open foundation and fine-tuned chat models},
  author={Touvron, Hugo and Martin, Louis and Stone, Kevin and Albert, Peter and Almahairi, Amjad and Babaei, Yasmine and Bashlykov, Nikolay and Batra, Soumya and Bhargava, Prajjwal and Bhosale, Shruti and others},
  journal={arXiv preprint arXiv:2307.09288},
  year={2023}
}

@article{dettmers2023qlora,
  title={Qlora: Efficient finetuning of quantized llms},
  author={Dettmers, Tim and Pagnoni, Artidoro and Holtzman, Ari and Zettlemoyer, Luke},
  journal={arXiv preprint arXiv:2305.14314},
  year={2023}
}

@inproceedings{wang2018glue,
  title={GLUE: A Multi-Task Benchmark and Analysis Platform for Natural Language Understanding},
  author={Wang, Alex and Singh, Amanpreet and Michael, Julian and Hill, Felix and Levy, Omer and Bowman, Samuel R},
  booktitle={International Conference on Learning Representations},
  year={2018}
}

@inproceedings{zhang2022adaptive,
  title={Adaptive Budget Allocation for Parameter-Efficient Fine-Tuning},
  author={Zhang, Qingru and Chen, Minshuo and Bukharin, Alexander and He, Pengcheng and Cheng, Yu and Chen, Weizhu and Zhao, Tuo},
  booktitle={The Eleventh International Conference on Learning Representations},
  year={2022}
}

@article{lialin2023scaling,
  title={Scaling down to scale up: A guide to parameter-efficient fine-tuning},
  author={Lialin, Vladislav and Deshpande, Vijeta and Rumshisky, Anna},
  journal={arXiv preprint arXiv:2303.15647},
  year={2023}
}

@inproceedings{hu2021lora,
  title={LoRA: Low-Rank Adaptation of Large Language Models},
  author={Hu, Edward J and Wallis, Phillip and Allen-Zhu, Zeyuan and Li, Yuanzhi and Wang, Shean and Wang, Lu and Chen, Weizhu and others},
  booktitle={International Conference on Learning Representations},
  year={2021}
}

@inproceedings{aghajanyan2021intrinsic,
  title={Intrinsic Dimensionality Explains the Effectiveness of Language Model Fine-Tuning},
  author={Aghajanyan, Armen and Gupta, Sonal and Zettlemoyer, Luke},
  booktitle={Proceedings of the 59th Annual Meeting of the Association for Computational Linguistics and the 11th International Joint Conference on Natural Language Processing (Volume 1: Long Papers)},
  pages={7319--7328},
  year={2021}
}

@inproceedings{valipour2023dylora,
  title={DyLoRA: Parameter-Efficient Tuning of Pre-trained Models using Dynamic Search-Free Low-Rank Adaptation},
  author={Valipour, Mojtaba and Rezagholizadeh, Mehdi and Kobyzev, Ivan and Ghodsi, Ali},
  booktitle={Proceedings of the 17th Conference of the European Chapter of the Association for Computational Linguistics},
  pages={3266--3279},
  year={2023}
}

@article{han2024parameter,
  title={Parameter-efficient fine-tuning for large models: A comprehensive survey},
  author={Han, Zeyu and Gao, Chao and Liu, Jinyang and Zhang, Sai Qian and others},
  journal={arXiv preprint arXiv:2403.14608},
  year={2024}
}

@article{gaoparameter,
  title={Parameter-Efficient Fine-Tuning with Discrete Fourier Transform},
  author={Gao, Ziqi and Wang, Qichao and Chen, Aochuan and Liu, Zijing and Wu, Bingzhe and Chen, Liang and Li, Jia},
  journal={arXiv preprint arXiv:2405.03003},
  year={2024}
}

@article{meng2024pissa,
  title={Pissa: Principal singular values and singular vectors adaptation of large language models},
  author={Meng, Fanxu and Wang, Zhaohui and Zhang, Muhan},
  journal={arXiv preprint arXiv:2404.02948},
  year={2024}
}

@article{wang2024milora,
  title={MiLoRA: Harnessing Minor Singular Components for Parameter-Efficient LLM Finetuning},
  author={Wang, Hanqing and Xiao, Zeguan and Li, Yixia and Wang, Shuo and Chen, Guanhua and Chen, Yun},
  journal={arXiv preprint arXiv:2406.09044},
  year={2024}
}

@article{team2024gemma,
  title={Gemma: Open models based on gemini research and technology},
  author={{Gemma Team}},
  journal={arXiv preprint arXiv:2403.08295},
  year={2024}
}

@inproceedings{liu2024dora,
  title={Dora: Weight-decomposed low-rank adaptation},
  author={Liu, Shih-Yang and Wang, Chien-Yi and Yin, Hongxu and Molchanov, Pavlo and Wang, Yu-Chiang Frank and Cheng, Kwang-Ting and Chen, Min-Hung},
  booktitle={Forty-first International Conference on Machine Learning},
  year={2024}
}

@article{masoudnia2014mixture,
  title={Mixture of experts: a literature survey},
  author={Masoudnia, Saeed and Ebrahimpour, Reza},
  journal={Artificial Intelligence Review},
  volume={42},
  number={2},
  pages={275--293},
  year={2014},
  publisher={Springer}
}

@book{davis2012fourier,
  title={Fourier series and orthogonal functions},
  author={Davis, Harry F},
  year={2012},
  publisher={Courier Corporation}
}

@article{zi2023delta,
  title={Delta-lora: Fine-tuning high-rank parameters with the delta of low-rank matrices},
  author={Zi, Bojia and Qi, Xianbiao and Wang, Lingzhi and Wang, Jianan and Wong, Kam-Fai and Zhang, Lei},
  journal={arXiv preprint arXiv:2309.02411},
  year={2023}
}

@article{dubey2024llama,
  title={The llama 3 herd of models},
  author={Dubey, Abhimanyu and Jauhri, Abhinav and Pandey, Abhinav and Kadian, Abhishek and Al-Dahle, Ahmad and Letman, Aiesha and Mathur, Akhil and Schelten, Alan and Yang, Amy and Fan, Angela and others},
  journal={arXiv preprint arXiv:2407.21783},
  year={2024}
}

@article{tian2024hydralora,
  title={HydraLoRA: An Asymmetric LoRA Architecture for Efficient Fine-Tuning},
  author={Tian, Chunlin and Shi, Zhan and Guo, Zhijiang and Li, Li and Xu, Chengzhong},
  journal={arXiv preprint arXiv:2404.19245},
  year={2024}
}

@article{li2024mixlora,
  title={Mixlora: Enhancing large language models fine-tuning with lora based mixture of experts},
  author={Li, Dengchun and Ma, Yingzi and Wang, Naizheng and Cheng, Zhiyuan and Duan, Lei and Zuo, Jie and Yang, Cal and Tang, Mingjie},
  journal={arXiv preprint arXiv:2404.15159},
  year={2024}
}

@article{feng2024mixture,
  title={Mixture-of-loras: An efficient multitask tuning for large language models},
  author={Feng, Wenfeng and Hao, Chuzhan and Zhang, Yuewei and Han, Yu and Wang, Hao},
  journal={arXiv preprint arXiv:2403.03432},
  year={2024}
}

@article{cai2024survey,
  title={A survey on mixture of experts},
  author={Cai, Weilin and Jiang, Juyong and Wang, Fan and Tang, Jing and Kim, Sunghun and Huang, Jiayi},
  journal={arXiv preprint arXiv:2407.06204},
  year={2024}
}

@article{zadouri2023pushing,
  title={Pushing mixture of experts to the limit: Extremely parameter efficient moe for instruction tuning},
  author={Zadouri, Ted and {\"U}st{\"u}n, Ahmet and Ahmadian, Arash and Ermi{\c{s}}, Beyza and Locatelli, Acyr and Hooker, Sara},
  journal={arXiv preprint arXiv:2309.05444},
  year={2023}
}

@inproceedings{liu2024moe,
  title={When MOE Meets LLMs: Parameter Efficient Fine-tuning for Multi-task Medical Applications},
  author={Liu, Qidong and Wu, Xian and Zhao, Xiangyu and Zhu, Yuanshao and Xu, Derong and Tian, Feng and Zheng, Yefeng},
  booktitle={Proceedings of the 47th International ACM SIGIR Conference on Research and Development in Information Retrieval},
  pages={1104--1114},
  year={2024}
}

@inproceedings{lee2022fnet,
  title={FNet: Mixing Tokens with Fourier Transforms},
  author={Lee-Thorp, James and Ainslie, Joshua and Eckstein, Ilya and Ontanon, Santiago},
  booktitle={Proceedings of the 2022 Conference of the North American Chapter of the Association for Computational Linguistics: Human Language Technologies},
  pages={4296--4313},
  year={2022}
}

@article{wang2022adamix,
  title={AdaMix: Mixture-of-adaptations for parameter-efficient model tuning},
  author={Wang, Yaqing and Agarwal, Sahaj and Mukherjee, Subhabrata and Liu, Xiaodong and Gao, Jing and Awadallah, Ahmed Hassan and Gao, Jianfeng},
  journal={arXiv preprint arXiv:2205.12410},
  year={2022}
}

@inproceedings{zhuang2022long,
  title={Long-range sequence modeling with predictable sparse attention},
  author={Zhuang, Yimeng and Zhang, Jing and Tu, Mei},
  booktitle={Proceedings of the 60th Annual Meeting of the Association for Computational Linguistics (Volume 1: Long Papers)},
  pages={234--243},
  year={2022}
}

@article{lepikhin2020gshard,
  title={{Gshard: Scaling giant models with conditional computation and automatic sharding}},
  author={Lepikhin, Dmitry and Lee, HyoukJoong and Xu, Yuanzhong and Chen, Dehao and Firat, Orhan and Huang, Yanping and Krikun, Maxim and Shazeer, Noam and Chen, Zhifeng},
  journal={arXiv preprint arXiv:2006.16668},
  year={2020}
}

@inproceedings{park2025llamaduo,
  title={Llamaduo: Llmops pipeline for seamless migration from service llms to small-scale local llms},
  author={Park, Chansung and Jiang, Juyong and Wang, Fan and Paul, Sayak and Tang, Jing},
  booktitle={Proceedings of the 63rd Annual Meeting of the Association for Computational Linguistics (Volume 1: Long Papers)},
  pages={33194--33215},
  year={2025}
}

@article{fedus2022switch,
  title={{Switch transformers: Scaling to trillion parameter models with simple and efficient sparsity}},
  author={Fedus, William and Zoph, Barret and Shazeer, Noam},
  journal={Journal of Machine Learning Research},
  volume={23},
  number={120},
  pages={1--39},
  year={2022}
}

@book{filter,
  title={Digital image processing},
  author={Gonzales, Rafael C and Wintz, Paul},
  year={1987},
  publisher={Addison-Wesley Longman Publishing Co., Inc.}
}

@article{gao2024higher,
  title={Higher layers need more lora experts},
  author={Gao, Chongyang and Chen, Kezhen and Rao, Jinmeng and Sun, Baochen and Liu, Ruibo and Peng, Daiyi and Zhang, Yawen and Guo, Xiaoyuan and Yang, Jie and Subrahmanian, VS},
  journal={arXiv preprint arXiv:2402.08562},
  year={2024}
}

@inproceedings{clark2019boolq,
  title={BoolQ: Exploring the Surprising Difficulty of Natural Yes/No Questions},
  author={Clark, Christopher and Lee, Kenton and Chang, Ming-Wei and Kwiatkowski, Tom and Collins, Michael and Toutanova, Kristina},
  booktitle={Proceedings of the 2019 Conference of the North American Chapter of the Association for Computational Linguistics: Human Language Technologies, Volume 1 (Long and Short Papers)},
  pages={2924--2936},
  year={2019}
}

@inproceedings{hu2023llm,
  title={LLM-Adapters: An Adapter Family for Parameter-Efficient Fine-Tuning of Large Language Models},
  author={Hu, Zhiqiang and Wang, Lei and Lan, Yihuai and Xu, Wanyu and Lim, Ee-Peng and Bing, Lidong and Xu, Xing and Poria, Soujanya and Lee, Roy},
  booktitle={Proceedings of the 2023 Conference on Empirical Methods in Natural Language Processing},
  pages={5254--5276},
  year={2023}
}

@article{clark2018think,
  title={Think you have solved question answering? try arc, the ai2 reasoning challenge},
  author={Clark, Peter and Cowhey, Isaac and Etzioni, Oren and Khot, Tushar and Sabharwal, Ashish and Schoenick, Carissa and Tafjord, Oyvind},
  journal={arXiv preprint arXiv:1803.05457},
  year={2018}
}

@inproceedings{bisk2020piqa,
  title={Piqa: Reasoning about physical commonsense in natural language},
  author={Bisk, Yonatan and Zellers, Rowan and Gao, Jianfeng and Choi, Yejin and others},
  booktitle={Proceedings of the AAAI conference on artificial intelligence},
  volume={34},
  number={05},
  pages={7432--7439},
  year={2020}
}

@inproceedings{mihaylov2018can,
  title={Can a Suit of Armor Conduct Electricity? A New Dataset for Open Book Question Answering},
  author={Mihaylov, Todor and Clark, Peter and Khot, Tushar and Sabharwal, Ashish},
  booktitle={Proceedings of the 2018 Conference on Empirical Methods in Natural Language Processing},
  pages={2381--2391},
  year={2018}
}

@article{sakaguchi2021winogrande,
  title={Winogrande: An adversarial winograd schema challenge at scale},
  author={Sakaguchi, Keisuke and Bras, Ronan Le and Bhagavatula, Chandra and Choi, Yejin},
  journal={Communications of the ACM},
  volume={64},
  number={9},
  pages={99--106},
  year={2021},
  publisher={ACM New York, NY, USA}
}

@inproceedings{zellers2019hellaswag,
  title={HellaSwag: Can a Machine Really Finish Your Sentence?},
  author={Zellers, Rowan and Holtzman, Ari and Bisk, Yonatan and Farhadi, Ali and Choi, Yejin},
  booktitle={Proceedings of the 57th Annual Meeting of the Association for Computational Linguistics},
  pages={4791--4800},
  year={2019}
}

@inproceedings{sap2019social,
  title={Social IQa: Commonsense Reasoning about Social Interactions},
  author={Sap, Maarten and Rashkin, Hannah and Chen, Derek and Le Bras, Ronan and Choi, Yejin},
  booktitle={Proceedings of the 2019 Conference on Empirical Methods in Natural Language Processing and the 9th International Joint Conference on Natural Language Processing (EMNLP-IJCNLP)},
  pages={4463--4473},
  year={2019}
}

@article{cobbe2021training,
  title={Training verifiers to solve math word problems},
  author={Cobbe, Karl and Kosaraju, Vineet and Bavarian, Mohammad and Chen, Mark and Jun, Heewoo and Kaiser, Lukasz and Plappert, Matthias and Tworek, Jerry and Hilton, Jacob and Nakano, Reiichiro and others},
  journal={arXiv preprint arXiv:2110.14168},
  year={2021}
}

@article{patel2021nlp,
  title={Are NLP models really able to solve simple math word problems?},
  author={Patel, Arkil and Bhattamishra, Satwik and Goyal, Navin},
  journal={arXiv preprint arXiv:2103.07191},
  year={2021}
}

@article{roy2016solving,
  title={Solving general arithmetic word problems},
  author={Roy, Subhro and Roth, Dan},
  journal={arXiv preprint arXiv:1608.01413},
  year={2016}
}

@inproceedings{hosseini2014learning,
  title={Learning to solve arithmetic word problems with verb categorization},
  author={Hosseini, Mohammad Javad and Hajishirzi, Hannaneh and Etzioni, Oren and Kushman, Nate},
  booktitle={Proceedings of the 2014 Conference on Empirical Methods in Natural Language Processing (EMNLP)},
  pages={523--533},
  year={2014}
}

@article{ling2017program,
  title={Program induction by rationale generation: Learning to solve and explain algebraic word problems},
  author={Ling, Wang and Yogatama, Dani and Dyer, Chris and Blunsom, Phil},
  journal={arXiv preprint arXiv:1705.04146},
  year={2017}
}

@article{koncel2015parsing,
  title={Parsing algebraic word problems into equations},
  author={Koncel-Kedziorski, Rik and Hajishirzi, Hannaneh and Sabharwal, Ashish and Etzioni, Oren and Ang, Siena Dumas},
  journal={Transactions of the Association for Computational Linguistics},
  volume={3},
  pages={585--597},
  year={2015},
  publisher={MIT Press One Rogers Street, Cambridge, MA 02142-1209, USA journals-info~…}
}

@article{wang2024kasa,
  title={KaSA: Knowledge-Aware Singular-Value Adaptation of Large Language Models},
  author={Wang, Fan and Jiang, Juyong and Park, Chansung and Kim, Sunghun and Tang, Jing},
  journal={arXiv preprint arXiv:2412.06071},
  year={2024}
}

@article{dou2023loramoe,
  title={Loramoe: Revolutionizing mixture of experts for maintaining world knowledge in language model alignment},
  author={Dou, Shihan and Zhou, Enyu and Liu, Yan and Gao, Songyang and Zhao, Jun and Shen, Wei and Zhou, Yuhao and Xi, Zhiheng and Wang, Xiao and Fan, Xiaoran and others},
  journal={arXiv preprint arXiv:2312.09979},
  volume={4},
  number={7},
  year={2023}
}

@article{fan2025make,
  title={Make LoRA Great Again: Boosting LoRA with Adaptive Singular Values and Mixture-of-Experts Optimization Alignment},
  author={Fan, Chenghao and Lu, Zhenyi and Liu, Sichen and Gu, Chengfeng and Qu, Xiaoye and Wei, Wei and Cheng, Yu},
  journal={arXiv preprint arXiv:2502.16894},
  year={2025}
}

@article{li2025beyond,
  title={Beyond higher rank: Token-wise input-output projections for efficient low-rank adaptation},
  author={Li, Shiwei and Luo, Xiandi and Wang, Haozhao and Tang, Xing and Cui, Ziqiang and Liu, Dugang and Li, Yuhua and He, Xiuqiang and Li, Ruixuan},
  journal={arXiv preprint arXiv:2510.23123},
  year={2025}
}

@article{zhang2025beyond,
  title={Beyond the time domain: Recent advances on frequency transforms in time series analysis},
  author={Zhang, Qianru and Yang, Peng and Wen, Honggang and Li, Xinzhu and Wang, Haixin and Sun, Fang and Song, Zezheng and Lai, Zhichen and Ma, Rui and Han, Ruihua and others},
  journal={arXiv preprint arXiv:2504.07099},
  year={2025}
}

@article{ma2025fouriercompress,
  title={FourierCompress: Layer-Aware Spectral Activation Compression for Efficient and Accurate Collaborative LLM Inference},
  author={Ma, Jian and Lyu, Xinchen and Jiang, Jun and Zou, Longhao and Ren, Chenshan and Cui, Qimei and Tao, Xiaofeng},
  journal={arXiv preprint arXiv:2510.16418},
  year={2025}
}

@article{he2023fourier,
  title={Fourier transformer: Fast long range modeling by removing sequence redundancy with fft operator},
  author={He, Ziwei and Yang, Meng and Feng, Minwei and Yin, Jingcheng and Wang, Xinbing and Leng, Jingwen and Lin, Zhouhan},
  journal={arXiv preprint arXiv:2305.15099},
  year={2023}
}

@inproceedings{fan2025towards,
  title={Towards More Efficient Post-training via Fourier Domain Adapter Framework},
  author={Fan, Yijia and Zhang, Jusheng and Wang, Keze},
  booktitle={Findings of the Association for Computational Linguistics: EMNLP 2025},
  pages={6175--6193},
  year={2025}
}

@inproceedings{radford2021learning,
  title={Learning transferable visual models from natural language supervision},
  author={Radford, Alec and Kim, Jong Wook and Hallacy, Chris and Ramesh, Aditya and Goh, Gabriel and Agarwal, Sandhini and Sastry, Girish and Askell, Amanda and Mishkin, Pamela and Clark, Jack and others},
  booktitle={International conference on machine learning},
  pages={8748--8763},
  year={2021},
  organization={PmLR}
}

@inproceedings{krause20133d,
  title={3d object representations for fine-grained categorization},
  author={Krause, Jonathan and Stark, Michael and Deng, Jia and Fei-Fei, Li},
  booktitle={Proceedings of the IEEE international conference on computer vision workshops},
  pages={554--561},
  year={2013}
}

@inproceedings{cimpoi2014describing,
  title={Describing textures in the wild},
  author={Cimpoi, Mircea and Maji, Subhransu and Kokkinos, Iasonas and Mohamed, Sammy and Vedaldi, Andrea},
  booktitle={Proceedings of the IEEE conference on computer vision and pattern recognition},
  pages={3606--3613},
  year={2014}
}

@article{helber2019eurosat,
  title={Eurosat: A novel dataset and deep learning benchmark for land use and land cover classification},
  author={Helber, Patrick and Bischke, Benjamin and Dengel, Andreas and Borth, Damian},
  journal={IEEE Journal of Selected Topics in Applied Earth Observations and Remote Sensing},
  volume={12},
  number={7},
  pages={2217--2226},
  year={2019},
  publisher={IEEE}
}

@inproceedings{houben2013detection,
  title={Detection of traffic signs in real-world images: The German Traffic Sign Detection Benchmark},
  author={Houben, Sebastian and Stallkamp, Johannes and Salmen, Jan and Schlipsing, Marc and Igel, Christian},
  booktitle={The 2013 international joint conference on neural networks (IJCNN)},
  pages={1--8},
  year={2013},
  organization={Ieee}
}

@article{cheng2017remote,
  title={Remote sensing image scene classification: Benchmark and state of the art},
  author={Cheng, Gong and Han, Junwei and Lu, Xiaoqiang},
  journal={Proceedings of the IEEE},
  volume={105},
  number={10},
  pages={1865--1883},
  year={2017},
  publisher={IEEE}
}

@inproceedings{xiao2010sun,
  title={Sun database: Large-scale scene recognition from abbey to zoo},
  author={Xiao, Jianxiong and Hays, James and Ehinger, Krista A and Oliva, Aude and Torralba, Antonio},
  booktitle={2010 IEEE computer society conference on computer vision and pattern recognition},
  pages={3485--3492},
  year={2010},
  organization={IEEE}
}

@inproceedings{netzer2011reading,
  title={Reading digits in natural images with unsupervised feature learning},
  author={Netzer, Yuval and Wang, Tao and Coates, Adam and Bissacco, Alessandro and Wu, Baolin and Ng, Andrew Y and others},
  booktitle={NIPS workshop on deep learning and unsupervised feature learning},
  volume={2011},
  number={5},
  pages={7},
  year={2011},
  organization={Granada}
}

@article{yang2025qwen3,
  title={Qwen3 technical report},
  author={Yang, An and Li, Anfeng and Yang, Baosong and Zhang, Beichen and Hui, Binyuan and Zheng, Bo and Yu, Bowen and Gao, Chang and Huang, Chengen and Lv, Chenxu and others},
  journal={arXiv preprint arXiv:2505.09388},
  year={2025}
}

@article{grattafiori2024llama,
  title={The llama 3 herd of models},
  author={Grattafiori, Aaron and Dubey, Abhimanyu and Jauhri, Abhinav and Pandey, Abhinav and Kadian, Abhishek and Al-Dahle, Ahmad and Letman, Aiesha and Mathur, Akhil and Schelten, Alan and Vaughan, Alex and others},
  journal={arXiv preprint arXiv:2407.21783},
  year={2024}
}

@article{liu2024adamole,
  title={Adamole: Fine-tuning large language models with adaptive mixture of low-rank adaptation experts},
  author={Liu, Zefang and Luo, Jiahua},
  journal={arXiv preprint arXiv:2405.00361},
  year={2024}
}

@inproceedings{si2025unleashing,
title={Unleashing the Power of Task-Specific Directions in Parameter Efficient Fine-tuning},
author={Chongjie Si and Zhiyi Shi and Shifan Zhang and Xiaokang Yang and Hanspeter Pfister and Wei Shen},
booktitle={The Thirteenth International Conference on Learning Representations},
year={2025},
url={https://openreview.net/forum?id=RYrJqz44p4}
}

@article{zhong2024neat,
  title={Neat: Nonlinear parameter-efficient adaptation of pre-trained models},
  author={Zhong, Yibo and Jiang, Haoxiang and Li, Lincan and Nakada, Ryumei and Liu, Tianci and Zhang, Linjun and Yao, Huaxiu and Wang, Haoyu},
  journal={arXiv preprint arXiv:2410.01870},
  year={2024}
}

@article{yu2020gradient,
  title={Gradient surgery for multi-task learning},
  author={Yu, Tianhe and Kumar, Saurabh and Gupta, Abhishek and Levine, Sergey and Hausman, Karol and Finn, Chelsea},
  journal={Advances in neural information processing systems},
  volume={33},
  pages={5824--5836},
  year={2020}
}

@inproceedings{kim2025lfma,
title={{LFMA}: Parameter-Efficient Fine-Tuning via Layerwise Fourier Masked Adapter with Top-k Frequency Selection},
author={Soo Yong Kim},
booktitle={NeurIPS 2025 Workshop on Symmetry and Geometry in Neural Representations},
year={2025},
url={https://openreview.net/forum?id=dBRaAOncLD}
}

@article{sun2025stronger,
  title={A Stronger Mixture of Low-Rank Experts for Fine-Tuning Foundation Models},
  author={Sun, Mengyang and Wang, Yihao and Feng, Tao and Zhang, Dan and Zhu, Yifan and Tang, Jie},
  journal={arXiv preprint arXiv:2502.15828},
  year={2025}
}

@article{zhang2025more,
  title={MoRE: A Mixture of Low-Rank Experts for Adaptive Multi-Task Learning},
  author={Zhang, Dacao and Zhang, Kun and Chu, Shimao and Wu, Le and Li, Xin and Wei, Si},
  journal={arXiv preprint arXiv:2505.22694},
  year={2025}
}

@article{guibas2021adaptive,
  title={Adaptive fourier neural operators: Efficient token mixers for transformers},
  author={Guibas, John and Mardani, Morteza and Li, Zongyi and Tao, Andrew and Anandkumar, Anima and Catanzaro, Bryan},
  journal={arXiv preprint arXiv:2111.13587},
  year={2021}
}

@book{oppenheim1999discrete,
  title={Discrete-time signal processing},
  author={Oppenheim, Alan V},
  year={1999},
  publisher={Pearson Education India}
}

@article{liu2019roberta,
  title={Roberta: A robustly optimized bert pretraining approach},
  author={Liu, Yinhan and Ott, Myle and Goyal, Naman and Du, Jingfei and Joshi, Mandar and Chen, Danqi and Levy, Omer and Lewis, Mike and Zettlemoyer, Luke and Stoyanov, Veselin},
  journal={arXiv preprint arXiv:1907.11692},
  year={2019}
}

@article{kalajdzievski2023rank,
  title={A rank stabilization scaling factor for fine-tuning with lora},
  author={Kalajdzievski, Damjan},
  journal={arXiv preprint arXiv:2312.03732},
  year={2023}
}

@inproceedings{ren2024melora,
  title={Melora: Mini-ensemble low-rank adapters for parameter-efficient fine-tuning},
  author={Ren, Pengjie and Shi, Chengshun and Wu, Shiguang and Zhang, Mengqi and Ren, Zhaochun and de Rijke, Maarten and Chen, Zhumin and Pei, Jiahuan},
  booktitle={Proceedings of the 62nd Annual Meeting of the Association for Computational Linguistics (Volume 1: Long Papers)},
  pages={3052--3064},
  year={2024}
}
\bibliographystyle{icml2026}

\newpage
\appendix
\onecolumn
\section{Algorithm for FourierMoE} \label{sec:algorithm}

\begin{algorithm}[h]
\caption{FourierMoE: Fourier Mixture-of-Experts Adaptation}
\label{alg:fouriermoe}
\begin{algorithmic}[1]
\REQUIRE Pretrained weights $\mathbf{W}_0 \in \mathbb{R}^{M \times N}$, Training dataset $\mathcal{D}_{train}$, scaling factor $\eta$.
\REQUIRE Number of experts $Z$, active experts $k$, frequency count $n$, bandwidth $\mathcal{W}$.
\STATE \textbf{Initialize Router:} $\Phi \leftarrow$ random initialization.
\STATE \textbf{Initialize Experts:} For each expert $i \in \{1, \dots, Z\}$:
\STATE \quad Sample active frequency indices $\Omega_i = \{(u,v)\}_{1}^n$ via Gaussian distribution $P_i(u,v)$ (Eq. \ref{eq:gaussian}).
\STATE \quad Initialize complex coefficients $\Theta_i = \{ \mathbf{C}_i(u,v) \in \mathbb{C} \mid (u,v) \in \Omega_i \}$.
\WHILE{not converged}
    \STATE Sample batch of inputs $x \sim \mathcal{D}_{train}$.
    \STATE \textbf{1. Gating \& Routing:}
    \STATE \quad Compute gating scores: $g = \text{Softmax}(\mathcal{G}_{\Phi}(x))$.
    \STATE \quad Select Top-$k$ experts: $\mathcal{S}(x) = \text{Top-}k(g)$.
    \STATE \textbf{2. Spectral-to-Spatial Expert Construction:}
    \STATE \quad Initialize composite update $\Delta \mathbf{W} \leftarrow \mathbf{0}^{M \times N}$.
    \FOR{each selected expert $i \in \mathcal{S}(x)$}
        \STATE \quad Construct sparse spectral matrix $\mathbf{E}_i \in \mathbb{C}^{M \times N}$ from $\Theta_i$.
        \STATE \quad \textit{// Enforce Conjugate Symmetry (Theorem \ref{thm:symmetry}) for lossless real-valued adaptation}
        \STATE \quad $\mathbf{E}_i(-u, -v) \leftarrow \mathbf{E}_i^*(u, v), \quad \forall (u,v) \in \Omega_i$.
        \STATE \quad Compute spatial expert update via IDFT:
        \STATE \quad $\Delta \mathbf{W}_i \leftarrow \mathcal{F}^{-1}(\mathbf{E}_i)$ \quad (Eq. \ref{eq:reparam}).
        \STATE \quad Accumulate: $\Delta \mathbf{W} \leftarrow \Delta \mathbf{W} + g_i \cdot \Delta \mathbf{W}_i$.
    \ENDFOR
    \STATE \textbf{3. Forward Pass:}
    \STATE \quad Compute output: $h = \mathbf{W}_0 x + \eta \Delta \mathbf{W} x$ \quad (Eq. \ref{eq:optimization}).
    \STATE \textbf{4. Optimization:}
    \STATE \quad Calculate Loss: $\mathcal{L} = \mathcal{L}_{task}(h) + \lambda \mathcal{L}_{aux}(g)$ \quad (Eq. \ref{eq:loss} and Eq. \ref{eq:aux_loss}).
    \STATE \quad Update parameters $\{\Theta, \Phi\}$ via gradient descent.
\ENDWHILE
\ENSURE Adapted parameters $\{\Theta, \Phi\}$.
\end{algorithmic}
\end{algorithm}

\section{More Related Works}\label{sec:related_work}
\subsection{Parameter-efficient Fine-tuning}
The escalating scale of LLMs renders FFT, a commonly used adaptation paradigm, computationally prohibitive \citep{zi2023delta,han2024parameter}. 
PEFT addresses this bottleneck by freezing the pretrained backbone and updating only a minimal set of added or existing parameters \citep{lialin2023scaling}. For instance, LoRA \citep{hu2021lora} employs a pair of learnable low-rank matrices to approximate the learned weight updates. While effective for single-task adaptation, such single-PEFT methods that rely on a single tunable module often suffer from capacity constraints and task interference in multi-task scenarios due to limited parameter sharing across diverse or conflicting objectives \citep{li2024mixlora,tian2024hydralora}. To alleviate these limitations, recent works incorporate the MoE architecture into PEFT, giving rise to MoPE methods \citep{fan2025make,sun2025stronger,wang2022adamix,feng2024mixture,zadouri2023pushing}.  These approaches employ a router to dynamically select lightweight expert modules conditioned on input tokens \citep{cai2024survey}, enabling a more flexible allocation of task-specific capacity and potentially reducing task interference. However, existing MoPE methods primarily operate in the spatial domain where experts are constructed as parallel instantiations of lightweight PEFT modules. Such designs lack explicit mechanisms to encourage orthogonality or diversity among experts,  which can lead to structural redundancy \citep{gao2024higher,tian2024hydralora}. Moreover, maintaining multiple spatial experts incurs additional parameter overhead, which partially compromises the fundamental efficiency goals of PEFT. Unlike prior MoPE methods, FourierMoE utilizes a router to dispatch tokens to experts specialized in distinct frequency bands, with each expert learning conjugate-symmetric complex coefficients in the Fourier domain.

\subsection{Fourier Transform}
Fourier transform has provided a powerful tool for analyzing global correlations and structural patterns in the frequency domain \citep{zhang2025beyond,fan2025towards}. It has been widely explored and used in various domains, such as optimizing long-sequence modeling \citep{zhuang2022long,guibas2021adaptive,he2023fourier}, compressing activation \cite{ma2025fouriercompress}, and improving the efficiency of the Transformer architecture \citep{lee2022fnet}. Fourier transform and spectral adaptation have recently attracted increasing attention in the PEFT regime. FourierFT \citep{gaoparameter} pioneers this line of research by parameterizing weight updates in the spectral domain. It learns a set of coefficients within a sparse spectral matrix and reconstructs it via IDFT to obtain the weight updates. While achieving performance competitive with LoRA, FourierFT treats spectral components uniformly and does not explicitly account for the varying importance of different frequency bands, which may lead to suboptimal parameter allocation \citep{kim2025lfma}. To refine this, LFMA \citep{kim2025lfma} optimizes only the top-$k$ coefficient coordinates with the largest magnitudes. However, these methods rely on static coefficient locations that cannot adapt to the dynamic evolution of gradients during fine-tuning or the distinct frequency signatures inherent in input tokens. Furthermore, they truncate the imaginary components of the complex spectrum, restricting the representation power of the learned updates. In contrast to prior works, FourierMoE integrates a frequency-adaptive router with specialized experts, enabling dynamic, frequency-aware adaptation for models. Furthermore, by enforcing complex coefficient learning and conjugate symmetry, FourierMoE bridges the gap between spectral-domain efficiency and spatial-domain expressivity, achieving superior performance while maintaining high efficiency.

\section{Baselines} \label{sec:baselines}
To evaluate our method's effectiveness, we compare FourierMoE against FFT, single PEFT, and MoPE baselines, including: \\
\textbf{FFT} initializes the foundation model with its pretrained weights and updates all parameters. \\
\textbf{FFT MoE} fully fine-tunes the upcycled MoE, which transforms the pretrained dense model into a MoE architecture by duplicating pretrained weights to create multiple experts and integrating a learnable router. \\
\textbf{LoRA} \citep{hu2021lora} reparameterizes the weight updates into a pair of tunable low-rank matrices, which can be merged with the foundation model after fine-tuning. \\
\textbf{rsLoRA} \citep{kalajdzievski2023rank} analyzes the scaling factor of LoRA and proposes to stabilize LoRA's learning trajectory with a factor of $\frac{\alpha}{\sqrt{r}}$ instead of $\frac{\alpha}{r}$. \\
\textbf{DoRA} \citep{liu2024dora} factorizes pretrained weights into magnitude and direction, updating only the magnitude via low-rank adaptation while keeping the direction fixed. \\
\textbf{PiSSA} \citep{meng2024pissa} performs singular value decomposition (SVD) on pretrained weights to obtain a principled low-rank initialization for fine-tuning. \\
\textbf{MiLoRA} \citep{wang2024milora} also leverages SVD on pretrained weights, but freezes the principal components and fine-tunes only the minor low-rank components. \\
\textbf{KaSA} \citep{wang2024kasa} applies SVD to the pretrained weights, truncates the minor components, and subsequently learns the weight updates in the SVD space. \\
\textbf{LoRA-Dash} \citep{si2025unleashing} first runs a short LoRA warm-up to identify task-specific directions by projecting the learned updates onto the singular directions of pretrained weights, and then jointly learns the corresponding coordinate scalars together with the low-rank adapters during fine-tuning. \\
\textbf{NEAT} \citep{zhong2024neat} introduces a lightweight network that takes the pretrained model weights as input and learns a nonlinear transformation to approximate the weight updates during fine-tuning. \\
\textbf{TopLoRA} \citep{li2025beyond} dynamically adjusts the LoRA weights by using a token-wise diagonal matrix generated by an additional, tunable projector network. \\
\textbf{MELoRA} \citep{ren2024melora} concatenates multiple mini LoRA modules in parallel along the diagonal to construct a block diagonal LoRA matrix, achieving higher effective rank without additional parameter overhead. \\
\textbf{FourierFT} \citep{gaoparameter} learns the Fourier coefficients in the spectral domain, then converts them back to the time domain via the IDFT to produce model weight updates. \\
\textbf{MoLoRA} \citep{zadouri2023pushing} integrates the MoE framework with PEFT by replacing the experts (usually initialized as the feed-forward network copies) with lightweight LoRA modules, which are selectively activated for each token via a routing mechanism. \\
\textbf{AdaMoLE} \citep{liu2024adamole} replaces the static top-k selection in the LoRA-MoE framework with dynamic thresholding mechanism to adaptively activate the most relevant experts based on input context. \\
\textbf{HydraLoRA} \citep{tian2024hydralora} introduces an asymmetric LoRA-MoE framework, characterized by a shared $\mathbf{A}$ matrix to capture general knowledge and multiple task-specific matrices $\mathbf{B}$ that are dynamically combined through a routing mechanism. \\
\textbf{GOAT} \citep{fan2025make} proposes a SVD-structured MoE framework that adaptively activates the experts initialized with distinct singular-value segments derived from the pretrained weights. In addition, it employs a theoretical scaling scheme to align the low-rank adaptation trajectories with those of fully fine-tuned MoE.

\section{Details of Datasets and Benchmarks}\label{sec:statistics_dataset}
\subsection{Commonsense Reasoning Datasets}
For the commonsense reasoning tasks, we fine-tune LLMs on the Commonsense170K dataset \citep{hu2023llm}. It contains data samples from eight datasets: ARC-e, ARC-c \citep{clark2018think}, OBQA \citep{mihaylov2018can}, SIQA \citep{sap2019social}, WinoGrande \citep{sakaguchi2021winogrande}, HellaSwag \citep{zellers2019hellaswag}, BoolQ \citep{clark2019boolq}, and PIQA \citep{bisk2020piqa}. We evaluate the fine-tuned models on the test splits of each dataset and report accuracy. The statistics of the datasets are presented in Table \ref{tab:dataset_commonsense170k}.
\begin{table}[h]
\centering
\caption{Statistics of the sub-datasets comprising the Commonsense170K dataset used for multi-task fine-tuning.}
\resizebox{0.73\linewidth}{!}{
\begin{tabular}{l|lcccc}
\toprule
\textbf{Dataset} & \textbf{Domain} & \textbf{\# Train} & \textbf{\# Test} & \textbf{Task Type} & \textbf{Answer}\\
\midrule
ARC-E & Natural Science & 2.3K & 2.3k & Question Answering & Option \\
ARC-C & Natural Science & 1.1K & 1.1k & Question Answering & Option \\
BoolQ & Wikipedia & 9.4K & 3.2k & Text Classification & Yes/No\\
OpenBookQA & Science Facts & 5.0K & 500 & Question Answering & Option \\
PIQA & Physical Interaction & 16.1K & 1.8k & Question Answering & Option\\
SIQA & Social Interaction & 33.4K & 1.9k & Question Answering & Option\\
HellaSwag & Video Caption & 39.9K & 10k & Sentence Completion & Option \\
WinoGrande & Winograd Schemas & 63.2K & 1.2k & Fill in the Blank & Option \\
\bottomrule
\end{tabular}}
\label{tab:dataset_commonsense170k}
\vspace{-1em}
\end{table}

\subsection{Math Reasoning Datasets}\label{sec:math_reasoning_dataset}
We adopt the Math10K dataset \citep{hu2023llm} to adapt LLMs for math reasoning tasks. This dataset consists of training data from AQuA \citep{ling2017program} and GSM8K \citep{cobbe2021training}. In addition, we evaluate the fine-tuned models on six math benchmarks: GSM8K \citep{cobbe2021training}, SVAMP \citep{patel2021nlp}, MultiArith \citep{roy2016solving}, AddSub \citep{hosseini2014learning}, AQuA \citep{ling2017program}, and SingleEq \citep{koncel2015parsing}. We report accuracy for all benchmarks.
The statistics of the datasets and benchmarks are shown in Table \ref{tab:dataset_math10k}.
\begin{table}[h]
\centering
\caption{Statistics of sub-datasets comprising the Math10K dataset used for multi-task fine-tuning. In the Math10K dataset, only GSM8K and AQuA provide training samples. The training sample counts of other datasets are marked with the ``-" symbol, indicating that we do not use their training data for fine-tuning.}
\resizebox{0.95\linewidth}{!}{
\begin{tabular}{l|lcccc}
\toprule
\textbf{Dataset} & \textbf{Domain} & \textbf{\# Train} & \textbf{\# Test} & \textbf{Task Type} & \textbf{Answer}\\
\midrule
AddSub & Arithmetic Word Problems & - & 395 & Math Word Problem Solving & Number \\
AQuA & GRE/GMAT Math & 100K & 254 & Multiple-Choice Question Answering & Option \\
GSM8K & Grade School Math Word Problems & 8.8K & 1.3k & Math Word Problem Solving & Number \\
MultiArith & Multi-step Math Word Problems & - & 600 & Math Word Problem Solving & Number  \\
SingleEQ & Grade School Algebra Word Problems & - & 508 & Math Word Problem Solving & Number \\
SVAMP & Arithmetic Word Problems & - & 1k & Math Word Problem Solving & Number \\
\bottomrule
\end{tabular}}
\label{tab:dataset_math10k}
\vspace{-0.4em}
\end{table}

\subsection{GLUE Dataset}\label{sec:glue_benchmark}
For NLU tasks, we fine-tune RoBERTa-large on seven datasets from GLUE \citep{wang2018glue}, including CoLA, SST-2, MRPC, QQP, MNLI, QNLI, and RTE. We evaluate the fine-tuned model on the test splits of each dataset and report accuracy. The statistics of each dataset is shown in Table \ref{tab:dataset_glue}.
\begin{table*}[ht!]
\centering
\caption{Statistics of the seven GLUE datasets for single-task fine-tuning.}
\resizebox{0.73\linewidth}{!}{
\begin{tabular}{l|lcccc}
\toprule
\textbf{Dataset} & \textbf{Domain} & \textbf{\# Train} & \textbf{\# Test} & \textbf{Task Type} & \textbf{Answer} \\
\midrule
CoLA & Miscellaneous & 8.5K & 1.0k & Acceptability & Label Text \\
SST-2 & Movie Reviews & 67.3K & 1.8k & Sentiment Analysis & Label Text \\
MRPC & News & 3.6K & 1.7k & Paraphrase & Label Text \\ 
QQP & Social QA & 364K & 391K & Paraphrase & Label Text \\
MNLI & Miscellaneous & 393K & 19.6K & Natural Language Inference & Label Text \\
QNLI & Wikipedia & 105K & 5.4k & Natural Language Inference & Label Text \\
RTE & News \& Wikipedia & 2.4K & 3k & Natural Language Inference & Label Text \\
\bottomrule
\end{tabular}}
\label{tab:dataset_glue}
\end{table*}

\subsection{Image Classification Datasets}\label{sec:image_classification_benchmark}
For image classification tasks, we fine-tune CLIP ViT-B/32 on seven datasets, including Cars \citep{krause20133d}, DTD \citep{cimpoi2014describing}, EuroSAT \citep{helber2019eurosat}, GTSRB \citep{houben2013detection}, RESISC45 \citep{cheng2017remote}, SUN397 \citep{xiao2010sun}, and SVHN \citep{netzer2011reading}. We evaluate the fine-tuned model on the test splits of each dataset and report accuracy. The statistics of datasets and benchmarks are shown in Table \ref{tab:dataset_image_classification}
\begin{table*}[ht!]
\centering
\caption{Statistics of the seven image classification datasets for single-task fine-tuning.}
\resizebox{0.73\linewidth}{!}{
\begin{tabular}{l|lcccc}
\toprule
\textbf{Dataset} & \textbf{Domain} & \textbf{\# Train} & \textbf{\# Test} & \textbf{Task Type} & \textbf{Answer} \\
\midrule
StanfordCars & Automobiles & 8.1K & 8.0K & Image Classification & Class Label \\
DTD & Textures & 3.7K & 1.8K & Image Classification & Class Label \\
EuroSAT & Satellite & 21.6K & 2.7K & Image Classification & Class Label \\
GTSRB & Traffic Signs & 26.6K & 12.6K & Image Classification & Class Label \\
RESISC45 & Aerial Imagery & 18.9K & 6.3K & Image Classification & Class Label \\
SUN397 & Scenes & 76.1K & 21.7K & Image Classification & Class Label \\
SVHN & Street Digits & 73.2K & 26.0K & Image Classification & Class Label \\
\bottomrule
\end{tabular}}
\label{tab:dataset_image_classification}
\vspace{-1em}
\end{table*}

\section{Training Details}\label{sec:train_details}
We follow the prior works \citep{li2024mixlora} to adopt a multi-task training setup for commonsense and math reasoning tasks, where models are trained on a mixed dataset and evaluated on different benchmarks. In addition, we employ a single-task training setup for image classification and NLU tasks, in which models are trained and evaluated separately on each task. To ensure an optimal performance, we meticulously tune the hyperparameters, including the learning rate, batch size, number of epochs, scaling value $\eta$, and the load-balancing loss weight $\lambda$, and other hyperparameters. The full hyperparameter configurations for all experimental settings are provided: commonsense reasoning in Table \ref{tab:commensense_config}, math reasoning in Table \ref{tab:math_reasoning_config}, image classification in Table \ref{tab:image_classification_config}, and NLU in Table \ref{tab:nlu_config}.
\begin{table}[h]
\centering
\caption{Detailed configurations used for fine-tuning LLaMA-2 7B and Gemma 7B models on commonsense reasoning tasks.}
\resizebox{0.45\textwidth}{!}{%
\begin{tabular}{l|cc}
\toprule
\multirow{2}*{\textbf{Hyperparameters}} & \multicolumn{2}{c}{\textbf{Commonsense Reasoning}}  \\
\cmidrule(lr){2-3}
& \textbf{LLaMA-2 7B} & \textbf{Gemma 7B} \\
\midrule
             Optimizer & \multicolumn{2}{c}{AdamW} \\
             LR & 8e-2   & 1.5e-1 \\
             LR Scheduler & \multicolumn{2}{c}{Linear} \\
             Max Seq. Len. & \multicolumn{2}{c}{256} \\
             Batch Size & \multicolumn{2}{c}{16} \\
             Accumulation Steps & \multicolumn{2}{c}{16} \\
             Dropout & \multicolumn{2}{c}{0.05}\\
             Warmup Ratio  & 0.03   & 0.05 \\
             \# Epochs   & \multicolumn{2}{c}{1} \\
             Spectral Coefficients $n$   & \multicolumn{2}{c}{8192} \\
             Placement & \multicolumn{2}{c}{Q,K,V,Up,Down,Gate} \\
             Scaling Value $\eta$  & 96  & 128 \\
             Load-balancing Scaling $\lambda$ & 0.002 & 0.001 \\ 
             \# Experts & \multicolumn{2}{c}{8} \\
             Top-k & \multicolumn{2}{c}{4} \\
\bottomrule
\end{tabular}
}
\label{tab:commensense_config}
\vspace{-1em}
\end{table}
\begin{table}[h]
\centering
\caption{Detailed configurations used for fine-tuning LLaMA-3 8B and Qwen2.5-14B models on math reasoning tasks.}
\resizebox{0.45\textwidth}{!}{
\begin{tabular}{l|cc}
\toprule
\multirow{2}*{\textbf{Hyperparameters}} & \multicolumn{2}{c}{\textbf{Math Reasoning}}  \\
\cmidrule(lr){2-3}
& \textbf{LLaMA-3 8B} & \textbf{Qwen2.5-14B} \\
\midrule
             Optimizer & \multicolumn{2}{c}{AdamW} \\
             LR & 2e-4 & 1e-4  \\
             LR Scheduler & \multicolumn{2}{c}{Linear} \\
             Max Seq. Len. & \multicolumn{2}{c}{1024}    \\
             Batch Size & \multicolumn{2}{c}{8} \\
             Accumulation Steps & \multicolumn{2}{c}{16} \\
             Dropout & \multicolumn{2}{c}{0.05} \\ 
             Warmup Ratio  & \multicolumn{2}{c}{0.05}  \\ 
             \# Epochs  & \multicolumn{2}{c}{1}  \\
             Spectral Coefficients $n$  & \multicolumn{2}{c}{8192} \\
             Placement & \multicolumn{2}{c}{Q, K, V, O, Up, Down, Gate} \\ 
             Scaling Value $\eta$   & 128 & 2 \\
             Load-balancing Scaling $\lambda$ & 0.001 & 0.002 \\
             \# Experts & \multicolumn{2}{c}{8} \\
             Top-k & \multicolumn{2}{c}{2} \\
\bottomrule
\end{tabular}
}
\label{tab:math_reasoning_config}
\end{table}
\begin{table}[h] 
    \centering 
    \caption{Detailed configurations used for fine-tuning CLIP ViT-B/32 on image classification tasks.}
    \resizebox{0.76\textwidth}{!}{%
    \begin{tabular}{l|ccccccc}
    \toprule
    \textbf{Hyperparameters} & \textbf{Cars} & \textbf{DTD} & \textbf{EuroSAT} & \textbf{GTSRB} & \textbf{RESISC45} & \textbf{SUN397} & \textbf{SVHN} \\ 
    \midrule
             Optimizer & \multicolumn{7}{c}{AdamW} \\
             Expert LR & 8e-5 & 1e-2 & 5e-2 & 2e-1 & 5e-2 & 5e-2 & 1.5e-1  \\
             Router LR & 5e-5 & 1e-4 & 2e-4 & 2e-4 & 2e-4 & 2e-4 & 2e-4 \\
             Head LR & 3e-4 & 1e-3 & 1e-3 & 1e-3 & 5e-3 & 1e-3 & 1e-3\\
             LR Scheduler & \multicolumn{7}{c}{Cosine} \\
             Max Seq. Len. & \multicolumn{7}{c}{50} \\
             Batch Size & 512 & 128 & 512 & 512 & 512 & 512 & 512 \\
             Accumulation Steps & \multicolumn{7}{c}{1} \\
             Dropout & \multicolumn{7}{c}{0.0} \\
             Warmup Ratio  & 0.03 & 0.03 & 0.05 & 0.05 & 0.03 & 0.05 & 0.05 \\ 
             \# Epochs  & 280 & 100 & 30 & 50 & 45 & 20 & 20  \\
             Spectral Coefficients $n$  & \multicolumn{7}{c}{3008} \\
             Placement & \multicolumn{7}{c}{Q,V} \\
             Scaling Value $\eta$  & 192 & 32 & 32 & 256 & 32 & 32 & 192 \\
             Load-balancing Scaling $\lambda$ &  \multicolumn{7}{c}{1e-3} \\
             \# Experts & \multicolumn{7}{c}{8} \\
             Top-k & \multicolumn{7}{c}{2} \\
             \bottomrule
    \end{tabular}}
    \label{tab:image_classification_config}
    \vspace{-0.4em}
\end{table}
\begin{table}[h] 
    \centering 
    \caption{Detailed configurations used for fine-tuning RoBERTa-large on NLU tasks.}
    \resizebox{0.75\textwidth}{!}{%
    \begin{tabular}{l|ccccccc}
    \toprule
    \textbf{Hyperparameters} & \textbf{CoLA} & \textbf{SST-2} & \textbf{MRPC} & \textbf{QQP} & \textbf{MNLI} & \textbf{QNLI} & \textbf{RTE} \\ 
    \midrule
             Optimizer & \multicolumn{7}{c}{AdamW} \\
             Expert LR & 5e-2 & 1.2e-1 & 9e-2 & 1e-1 & 1e-5 & 9e-2 & 1e-1  \\
             Router LR & 1.5e-4 & 1e-4 & 2e-4 & 3.9e-4 & 2e-4 & 1.5e-4 & 1.5e-4 \\
             Head LR & 1e-2 & 5e-4 & 1e-3 & 8.4e-4 & 5e-3 & 1e-3 & 6.5e-3\\
             LR Scheduler & \multicolumn{7}{c}{Linear} \\
             Max Seq. Len. & 256 & 128 & 512 & 128 & 512 & 512 & 512  \\
             Batch Size & \multicolumn{7}{c}{32}  \\
             Accumulation Steps & \multicolumn{7}{c}{1}  \\
             Dropout & \multicolumn{7}{c}{0.0}  \\
             Warmup Ratio  & \multicolumn{7}{c}{0.06}  \\ 
             \# Epochs  & 80 & 10 & 30 & 15 & 10 & 30 & 60 \\ 
             Spectral Coefficients $n$  & \multicolumn{7}{c}{1008} \\
             Placement & \multicolumn{7}{c}{Q,V} \\ 
             Scaling Value $\eta$  & 115 & 138 & 160 & 64 & 96 & 72 & 136 \\
             Load-balancing Scaling $\lambda$ & 1.5e-3 & 1.5e-3 & 2e-3 & 2.5e-3 & 3.5e-3 & 2.5e-3 & 1.7e-3 \\
             \# Experts & \multicolumn{7}{c}{8} \\
             Top-k  & \multicolumn{7}{c}{2}  \\
             \bottomrule
    \end{tabular}}
    \label{tab:nlu_config}
\end{table}

\section{Additional Experimental Results}
\subsection{Effect of Frequency Bandwidth $\mathcal{W}$ on Expert Specialization} \label{sec:bandpass}
As defined in Eq. \ref{eq:gaussian}, each expert in FourierMoE is assigned within a specific frequency bandwidth $\mathcal{W}$, which controls the range of spectral components it can access and thus governs the degree of expert specialization. To explain the non-monotonic performance trend observed in Section \ref{sec:hyperparameter_sensitivity}, we analyze how varying $\mathcal{W}$ shapes the spectral organization of experts, as visualized in Figure \ref{fig:bandpss}. When $\mathcal{W}=0.06$, experts are restricted to narrow low-frequency bands, preventing them from modeling informative high-frequency components and consequently limiting per-expert expressiveness. Increasing $\mathcal{W}$ to $0.12$ expands frequency coverage while maintaining well-separated annular structures across experts. This regime strikes a favorable balance between expressiveness and specialization, enabling experts to capture richer spectral patterns without sacrificing functional diversity, which aligns with the performance reported in Table \ref{tab:hyper_band}. In contrast, further increasing the bandwidth to $0.24$ and beyond introduces spectral overlap across experts, weakening specialization and reducing parameter efficiency. In the extreme regime ($\mathcal{W} \geq 0.96$), expert spectra become highly overlapping and near-uniform, collapsing the MoE into a collection of functionally similar modules. Such loss of specialization can lead to ambiguous routing decisions and unstable optimization, ultimately degrading model performance. These observations reveal a trade-off in FourierMoE: while increasing $\mathcal{W}$ enhances per-expert expressiveness, excessively wide bandwidths undermine cross-expert spectral specialization.

\begin{figure*}[t!]
\centering
\includegraphics[width=0.95\linewidth]{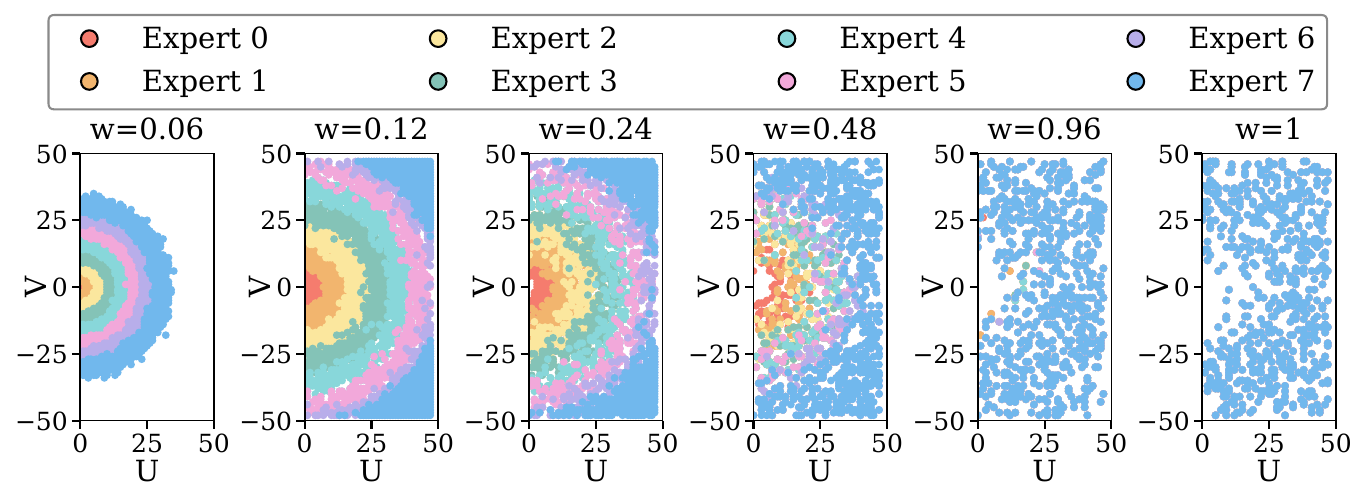}
\caption{Distributions of expert spectral coefficients under different frequency bandwidths $\mathcal{W}$. Moderate bandwidths expand per-expert frequency coverage while preserving well-separated spectral structures, whereas overly large $\mathcal{W}$ induces substantial overlap, leading to degraded specialization.}
\label{fig:bandpss}
\end{figure*}

\section{Complexity Analysis}
\label{subsec:complexity}

We analyze the computational and memory complexity of FourierMoE compared to standard FFT and LoRA fine-tuning. Let $M, N$ denote the weight matrix dimensions, $Z$ the total number of experts, $k$ the number of active experts, and $n$ the number of active spectral coefficients per expert (where $n \ll MN$).

\paragraph{Parameter Efficiency.}
Standard FFT requires updating $\mathcal{O}(MN)$ parameters per layer. LoRA reduces this to $\mathcal{O}(r(M+N))$ where $r$ is the rank of the lightweight tunable matrices. FourierMoE achieves sparsity by learning only the active spectral coefficients and the router. The total trainable parameter count is:
\begin{equation}
\mathcal{P}_{\text{FourierMoE}} = \mathcal{O}(Z \cdot n + Z \cdot N).
\end{equation}
Since the crucial adaptation information is concentrated in a small subset of dominant frequency components \citep{kim2025lfma,gaoparameter}, we can set $n$ such that $Z \cdot n \ll r(M+N)$, allowing us to scale the number of experts $Z$ without incurring significant memory overhead.

\paragraph{Inference Latency.}
The computational cost of FourierMoE during inference consists of the routing operation and the expert reconstruction.
\begin{itemize}
    \item \textbf{Routing Overhead:} The gating mechanism requires a lightweight projection $\mathcal{O}(N \cdot Z)$, which is negligible compared to the forward pass of the LLM layers.
    \item \textbf{Reconstruction Cost:} Constructing the update $\Delta \mathbf{W}$ requires summing $k$ experts. Utilizing the Fast Fourier Transform, the reconstruction complexity is $\mathcal{O}(k \cdot MN \log(MN))$.
\end{itemize}

Crucially, the inference cost depends only on the active set size $k$, not the total expert pool $Z$. This allows FourierMoE to increase model capacity (via larger $Z$) while maintaining constant-time inference complexity, satisfying the condition $\mathcal{O}(k) \approx \text{const}$ with respect to $Z$. Furthermore, by caching the reconstructed $\Delta \mathbf{W}$ for static tasks or utilizing parallelized GPU-optimized FFT kernels, the latency overhead remains within a tractable margin for real-time deployment.

\section{Spectral-Spatial Rank Duality}
\label{sec:theory_analysis}

In this section, we establish a rigorous theoretical connection between the spectral composition of our experts and the algebraic properties of the resulting weight updates in the spatial domain. Specifically, we investigate the relationship between the spectral sparsity of an expert $\mathbf{E}_i$ and the matrix rank of its spatial counterpart $\Delta \mathbf{W}_i$.

We demonstrate that FourierMoE generalizes low-rank adaptation methods (e.g., LoRA) by enabling dynamic rank allocation. Unlike static low-rank methods that fix the rank $r$ globally, FourierMoE allows the router to dynamically assign a specific rank capacity to each token by selecting experts with varying spectral densities.

\subsection{Rank Properties of the Fourier Basis}
\label{subsec:rank_properties}

To analyze the rank of the model's weight update $\Delta \mathbf{W}$, we examine the algebraic structure of the Fourier basis kernel $\mathbf{B}_{u,v}$ defined in Eq.~(\ref{eq:reparam}).

\begin{lemma}[\textbf{Rank-1 Property of Fourier Kernels}]
\label{lemma:rank1}
For any frequency coordinate $(u,v)$, the Fourier basis kernel $\mathbf{B}_{u,v} \in \mathbb{C}^{M \times N}$ is a rank-1 matrix.
\end{lemma}

\begin{proof}
Recall the definition of the kernel entry at spatial index $(q,y)$:
\begin{equation}
\mathbf{B}_{u,v}(q, y) = \exp\left(j 2\pi \frac{uq}{M}\right) \cdot \exp\left(j 2\pi \frac{vy}{N}\right).
\end{equation}
Let $\mathbf{f}_u \in \mathbb{C}^M$ and $\mathbf{g}_v \in \mathbb{C}^N$ be vectors such that the entries are defined as $\mathbf{f}_u(x) = e^{j 2\pi uq/M}$ and $\mathbf{g}_v(y) = e^{j 2\pi vy/N}$. We can express the matrix $\mathbf{B}_{u,v}$ as the outer product:
\begin{equation}
\mathbf{B}_{u,v} = \mathbf{f}_u \mathbf{g}_v^T.
\end{equation}
Since $\mathbf{B}_{u,v}$ is formed by the outer product of two non-zero vectors, it possesses exactly one non-zero singular value. Thus, $\operatorname{rank}(\mathbf{B}_{u,v}) = 1$.
\end{proof}

\subsection{Spectral Sparsity Bounds Spatial Rank}

Building on Lemma~\ref{lemma:rank1}, we derive the upper bound on the rank of the spatial update generated by an expert defined by a sparse set of active frequencies.

\begin{theorem}[\textbf{Spectral Sparsity-Rank Inequality}]
\label{thm:rank_bound}
Let $\mathbf{E}_i$ be a frequency-specialized expert with a set of active frequencies $\Omega_i$ satisfying the conjugate symmetry condition (Theorem~\ref{thm:symmetry}). Let $K = |\Omega_i|$ be the spectral sparsity (the total number of non-zero frequency coefficients). The rank of the resulting spatial update $\Delta \mathbf{W}_i$ is bounded by:
\begin{equation}
\operatorname{rank}(\Delta \mathbf{W}_i) \leq \min(M, N, K).
\end{equation}
\end{theorem}

\begin{proof}
From Eq.~(\ref{eq:reparam}), the spatial update is a linear combination of basis kernels:
\begin{equation}
\Delta \mathbf{W}_i = \frac{1}{MN} \sum_{(u,v) \in \Omega_i} \mathbf{F}(u,v) \cdot \mathbf{B}_{u,v}.
\end{equation}
Using the subadditivity property of matrix rank (the rank of a sum is less than or equal to the sum of the ranks), we have:
\begin{align}
\operatorname{rank}(\Delta \mathbf{W}_i) &= \operatorname{rank}\left( \sum_{(u,v) \in \Omega_i} c_{u,v} \mathbf{B}_{u,v} \right) \nonumber \\
&\leq \sum_{(u,v) \in \Omega_i} \operatorname{rank}(c_{u,v} \mathbf{B}_{u,v}).
\end{align}
By Lemma~\ref{lemma:rank1}, $\operatorname{rank}(\mathbf{B}_{u,v}) = 1$. Therefore:
\begin{equation}
\operatorname{rank}(\Delta \mathbf{W}_i) \leq \sum_{(u,v) \in \Omega_i} 1 = |\Omega_i| = K.
\end{equation}
Since the rank of any matrix in $\mathbb{R}^{M \times N}$ cannot exceed its dimensions, it holds that $\operatorname{rank}(\Delta \mathbf{W}_i) \leq \min(M, N, K)$.
\end{proof}

\paragraph{Remark on Real-Valued Updates.}
Due to the Conjugate Symmetry Condition (Theorem~\ref{thm:symmetry}), active frequencies must appear in pairs $((u,v), (\langle -u \rangle_M, \langle -v \rangle_N))$ to ensure $\Delta \mathbf{W}_i \in \mathbb{R}^{M \times N}$. A symmetric pair combines to form a real-valued sinusoidal wave:
\begin{equation}
c \mathbf{B}_{u,v} + c^* \mathbf{B}_{\langle -u \rangle_M, \langle -v \rangle_N} = 2|c| \cos\left( 2\pi\left(\frac{uq}{M} + \frac{vy}{N}\right) + \angle c \right).
\end{equation}
While this sum is real-valued, it is formed by the sum of two rank-1 complex matrices. Thus, each symmetric pair contributes at most 2 to the rank of the spatial matrix.

\subsection{Input-Dependent Dynamic Rank Allocation}

The insights from Theorem~\ref{thm:rank_bound} reveal the fundamental mechanism of FourierMoE. By initializing experts with Gaussian filters of varying bandwidths $\mathcal{W}_i$ (Eq.~\ref{eq:gaussian}), we effectively create a bank of experts with different rank capacities.

\begin{corollary}[\textbf{Router as a Rank Selector}]
Let the router $\mathcal{G}(x)$ select a subset of experts $\mathcal{S}(x)$. The effective rank of the update applied to input $x$ is:
\begin{equation}
r_{eff}(x) = \operatorname{rank}\left( \sum_{i \in \mathcal{S}(x)} \mathcal{G}(x)_i \Delta \mathbf{W}_i \right) \leq \sum_{i \in \mathcal{S}(x)} |\Omega_i|.
\end{equation}
\end{corollary}

This leads to a critical theoretical distinction between FourierMoE and standard LoRA:
\begin{itemize}
    \item \textbf{LoRA (Static Rank):} Forces a fixed rank $r$ for all inputs, i.e., $\Delta \mathbf{W} = \mathbf{B}\mathbf{A}$. The expressivity is static regardless of input complexity.
    \item \textbf{FourierMoE (Dynamic Rank):} The router dynamically determines the required complexity. For ``easy'' tokens (e.g., stop words), the router may select Low-Bandwidth Experts (small $|\Omega_i|$), applying a \textit{Low-Rank} update. For ``hard'' tokens (e.g., reasoning steps), it may select High-Bandwidth Experts (large $|\Omega_i|$), applying a \textit{High-Rank} update.
\end{itemize}

\subsection{Orthogonality and Global Receptive Field}

Finally, we address why spectral rank-1 updates differ from spatial rank-1 updates. In standard low-rank adaptation, the basis vectors (columns of $\mathbf{A}$ and $\mathbf{B}$) are often learned to be sparse or localized in the spatial domain. In contrast, the Fourier basis vectors $\mathbf{f}_u$ and $\mathbf{g}_v$ are \textit{dense} in the spatial domain; every element has magnitude $1/\sqrt{M}$.

\begin{proposition}[\textbf{Global Information Flow}]
A rank-1 update in the Fourier domain effects a global update in the spatial domain. Specifically, modifying a single spectral coefficient $\mathbf{F}(u,v)$ updates every weight in $\Delta \mathbf{W}$ simultaneously with constant magnitude, merely shifting the phase.
\end{proposition}

This property implies that our FourierMoE is exceptionally efficient at learning \textit{global} features (via low frequencies) and \textit{distributed} patterns, whereas standard spatial sparsity methods struggle to propagate information globally without increasing parameter density. The Fourier transform acts as a maximally incoherent basis change, allowing sparse spectral experts to model dense spatial correlations efficiently.



\end{document}